\documentclass{article}
\usepackage[final,nonatbib]{neurips_modified}
 
 \usepackage[utf8]{inputenc} \usepackage{url}            \usepackage{booktabs}       \usepackage{nicefrac}       \usepackage{microtype}      

\usepackage{fancyhdr}
\usepackage[colorlinks,citecolor=blue,urlcolor=blue,linkcolor=blue,bookmarks=false]{hyperref}
\usepackage{amsfonts,epsfig,graphicx}
\usepackage{afterpage}
\usepackage{amsmath,amssymb,amsthm} 
\usepackage[T1]{fontenc} 
\usepackage{epsf} 
\usepackage{graphics} 
\usepackage{amsfonts,amsmath}
\usepackage[sort,numbers]{natbib} 
\usepackage{psfrag,xspace}
\usepackage{color,etoolbox}
\usepackage{subfigure} 
\usepackage{algpseudocode}
\usepackage{algorithm}
\usepackage{diagbox}

\usepackage{mathabx}
\usepackage[colorinlistoftodos]{todonotes}
\newtheorem{theorem}{Theorem}[section]
\newtheorem{remark}{Remark}[section]
\newtheorem{corollary}{Corollary}[section]
\newtheorem{lemma}{Lemma}[section]
\newtheorem{proposition}{Proposition}[section]
\theoremstyle{definition}
\newtheorem{defn}{Definition}[section]

\newcommand{\newref}[2][]{\hyperref[#2]{#1~\ref*{#2}}}
\renewcommand{\eqref}[1]{\hyperref[#1]{(\ref*{#1})}}
\clearpage{}

\newcommand{\vect}[1]{\ensuremath{\mathbf{#1}}}

\newcommand{\grad}{\nabla}

\newcommand{\argmin}{\mathop{\rm argmin}}
\newcommand{\argmax}{\mathop{\rm argmax}}

\newcommand{\iprod}[2]{\left\langle #1, #2 \right\rangle}

\newcommand{\pertdist}{P_{\text{PRTB}}}
\newcommand{\dom}[1]{\text{dom}(#1)}

\newcommand{\norm}[1]{\|{#1} \|}

\newcommand{\E}[1]{\mathbb{E}\left[#1\right]}
\newcommand{\Eover}[2]{\mathbb{E}_{#1}\left[#2\right]}

\renewcommand{\Pr}{\mathbb{P}}

\renewcommand{\u}{\vect{u}}

\newcommand{\x}{\vect{x}}

\newcommand{\y}{\vect{y}}
\newcommand{\z}{\vect{z}}

\newcommand{\cN}{\mathcal{N}}

\newcommand{\cF}{\mathcal{F}}

\newcommand{\cX}{\mathcal{X}}
\newcommand{\cY}{\mathcal{Y}}
\newcommand{\cP}{\mathcal{P}}
\newcommand{\cQ}{\mathcal{Q}}

\newcommand{\order}[1]{O\left({#1}\right)}
\newcommand{\torder}[1]{\Tilde{O}\left({#1}\right)}

\newcommand{\gphiNA}{\grad\Phi}
\newcommand{\gphi}[1]{\grad\Phi\left(#1\right)}
\newcommand{\gammaF}{\gamma_{\cF}}
\newcommand{\gammaFp}{\gamma_{\cF'}}
\newcommand{\normcompf}{\Psi_1}
\newcommand{\normcompb}{\Psi_2}
\newcommand{\stability}{C}
\clearpage{}

\title{Follow the Perturbed Leader: Optimism and Fast Parallel Algorithms for Smooth Minimax Games}

\author{Arun Sai Suggala \\
  Carnegie Mellon University\\
  \texttt{asuggala@andrew.cmu.edu} \\
   \And
   Praneeth Netrapalli\\
  Microsoft Research, India\\
  \texttt{praneeth@microsoft.com} \\
}

\begin{document}
\maketitle

\begin{abstract}
We consider the problem of online learning and its application to solving minimax games. For the online learning problem, Follow the Perturbed Leader (FTPL) is a widely studied algorithm which enjoys the optimal $\order{T^{1/2}}$ \emph{worst case} regret guarantee for both convex and nonconvex losses. In this work, we show that when the sequence of loss functions is \emph{predictable}, a simple modification of FTPL which incorporates optimism can achieve better regret guarantees, while retaining the optimal worst case regret guarantee for unpredictable sequences. A key challenge in obtaining these tighter regret bounds is the stochasticity and optimism in the algorithm, which requires different analysis techniques than those commonly used in the analysis of FTPL. The key ingredient we utilize in our analysis is the dual view of perturbation as regularization.
While our algorithm has several applications, we consider the specific application of minimax games. For solving  smooth convex-concave games, our algorithm only requires access to a linear optimization oracle. For Lipschitz and smooth nonconvex-nonconcave games, our algorithm requires access to an optimization oracle which computes the perturbed best response. In both these settings, our algorithm solves the game up to an accuracy of $\order{T^{-1/2}}$ using $T$ calls to the optimization oracle. An important feature of our algorithm is that it is highly parallelizable and requires only $O(T^{1/2})$ iterations, with each iteration making $\order{T^{1/2}}$ parallel calls to the optimization oracle.
\end{abstract}

\section{Introduction}
\label{sec:intro}
In this work, we consider the problem of online learning, where in each iteration, the learner chooses an action and observes a loss function. The goal of the learner is to choose a sequence of actions which minimizes the cumulative loss suffered over the course of learning. The paradigm of online learning has many theoretical and practical applications and has been widely studied in a number of fields, including game theory and machine learning. One of the popular applications of online learning is in solving minimax games arising in various contexts such as boosting~\citep{freund1996game}, robust optimization~\citep{chen2017robust}, Generative Adversarial Networks~\citep{goodfellow2014generative}.

In recent years, a number of efficient algorithms have been developed for regret minimization. These algorithms fall into two broad categories, namely, Follow the Regularized Leader (FTRL)~\citep{mcmahan2017survey} and FTPL~\citep{kalai2005efficient} style algorithms. When the sequence of loss functions encountered by the learner are convex, both these algorithms are known to achieve the optimal $\order{T^{1/2}}$ worst case regret~\citep{cesa2006prediction, hazan2016introduction}. While these algorithms have similar regret guarantees, they differ in computational aspects. Each iteration of FTRL involves implementation of an expensive projection step. In contrast, each step of FTPL involves solving a linear optimization problem, which can be implemented efficiently for many problems of interest~\citep{garber2013playing,gidel2016frank,hazan2020projection}. This crucial difference between FTRL and FTPL makes the latter algorithm more attractive in practice. Even in the more general nonconvex setting, where the loss functions encountered by the learner can potentially be nonconvex, FTPL algorithms are attractive. In this setting, FTPL requires access to an offline optimization oracle which computes the perturbed best response, and achieves $\order{T^{1/2}}$ worst case regret~\citep{suggala2019online}. Furthermore, these optimization oracles can be efficiently implemented for many problems by leveraging the rich body of work on global optimization~\cite{horst2013handbook}.

Despite its importance and popularity, FTPL has been mostly studied for the worst case setting, where the loss functions are assumed to be adversarially chosen. In a number of applications of online learning, the loss functions are actually benign and predictable~\citep{rakhlin2012online}. In such scenarios, FTPL can not utilize the predictability of losses to achieve tighter regret bounds. While~\citep{rakhlin2012online, suggala2019online} study variants of FTPL which can make use of predictability, these works either consider restricted settings or provide sub-optimal regret guarantees (see Section~\ref{sec:bg} for more details). This is unlike FTRL, where optimistic variants that can utilize the predictability of loss functions  have been well understood~\citep{rakhlin2012online, rakhlin2013optimization} and have been shown to provide faster convergence rates in applications such as minimax games. In this work, we aim to bridge this gap and study a variant of FTPL called Optimistic FTPL (OFTPL), which can achieve better regret bounds, while retaining the optimal worst case regret guarantee for unpredictable sequences. The main challenge in obtaining these tighter regret bounds is handling the stochasticity and optimism in the algorithm, which requires different analysis techniques to those commonly used in the analysis of FTPL. 
In this work, we rely on the dual view of perturbation as regularization to derive regret bounds of OFTPL. 

To demonstrate the usefulness of OFTPL, we consider the problem of solving minimax games. A widely used approach for solving such games relies on online learning algorithms~\citep{cesa2006prediction}. In this approach, both the minimization and the maximization players play a repeated game against each other and rely on online learning algorithms to choose their actions in each round of the game. In our algorithm for solving games, we let both the players use OFTPL to choose their actions. For solving  smooth convex-concave games, our algorithm only requires access to a linear optimization oracle. For Lipschitz and smooth nonconvex-nonconcave games, our algorithm requires access to an optimization oracle which computes the perturbed best response. In both these settings, our algorithm solves the game up to an accuracy of $\order{T^{-1/2}}$ using $T$ calls to the optimization oracle. While there are prior algorithms that achieve these convergence rates~\cite{he2015semi,suggala2019online}, an important feature of our algorithm is that it is highly parallelizable and requires only $O(T^{1/2})$ iterations, with each iteration making $\order{T^{1/2}}$ parallel calls to the optimization oracle. 
We note that such parallelizable algorithms are especially useful in large-scale machine learning applications such as training of GANs, adversarial training, which often involve huge datasets such as ImageNet~\citep{russakovsky2015imagenet}.
 
\section{Preliminaries and Background Material}
\label{sec:bg}
\vspace{-0.05in}
\paragraph{Online Learning.} The online learning framework can be seen as a repeated game between a learner and an adversary. 
In this framework, in each round $t$, the learner makes a prediction \mbox{$\x_t \in \cX \subseteq \mathbb{R}^d$} for some compact set $\cX$, and the adversary simultaneously chooses a loss function \mbox{$f_t:\cX \rightarrow \mathbb{R}$} and observe each others actions. 
The goal of the learner is to choose a sequence of actions $\{\x_t\}_{t=1}^T$ so that the following notion of regret is  minimized: \mbox{$\sum_{t = 1}^T f_t(\x_t) - \inf_{\x \in \cX}\sum_{t=1}^Tf_t(\x).$}

When the domain $\cX$ and loss functions $f_t$ are convex, a number of efficient algorithms for regret minimization have been studied. Some of these include deterministic algorithms such as Online Mirror Descent, Follow the Regularized Leader (FTRL)~\citep{hazan2016introduction, mcmahan2017survey}, and stochastic algorithms such as Follow the Perturbed Leader (FTPL)~\citep{kalai2005efficient}. In FTRL, one predicts $\x_t$ as $\argmin_{\x \in \cX} \sum_{i=1}^{t-1}\iprod{\grad_i}{\x} + R(\x)$, for some strongly convex regularizer $R$, where $\grad_i = \grad f_i(\x_i)$. FTRL is known to achieve the optimal $O(T^{1/2})$ worst case regret in the convex setting~\citep{mcmahan2017survey}. 
In FTPL, one predicts $\x_t$ as $m^{-1}\sum_{j=1}^m\x_{t,j}$, where $\x_{t,j}$ is a minimizer of the following linear optimization problem: $\argmin_{\x \in \cX} \iprod{\sum_{i=1}^{t-1}\grad_i - \sigma_{t,j}}{\x}.$ Here,  $\{\sigma_{t,j}\}_{j=1}^m$ are independent random perturbations drawn from some appropriate probability distribution such as exponential distribution or uniform distribution in a hyper-cube. Various choices of perturbation distribution gives rise to various FTPL algorithms. 
When the loss functions are linear,~\citet{kalai2005efficient} show that FTPL achieves $\order{T^{1/2}}$ expected regret, irrespective of the choice of $m$. When the loss functions are convex,~\citet{hazan2016introduction} showed that the deterministic version of FTPL (\emph{i.e.,} as $m \to \infty$) achieves $\order{T^{1/2}}$ regret. While projection free methods for online convex learning have been studied since the early work of~\cite{hazan2012projection}, surprisingly, regret bounds of FTPL for finite $m$ have only been recently studied~\citep{hazan2020projection}. \citet{hazan2020projection} show that for Lipschitz and convex functions, FTPL achieves $\order{T^{1/2} + m^{-1/2}T }$ expected regret, and for smooth convex functions, the algorithm achieves $\order{T^{1/2} + m^{-1}T}$ expected regret.

When either the domain $\cX$ or the loss functions $f_t$ are non-convex, no deterministic algorithm can achieve $o(T)$ regret~\citep{cesa2006prediction,suggala2019online}. In such cases, one has to rely on randomized algorithms to achieve sub-linear regret. In randomized algorithms, in each round $t$, the learner samples the prediction $\x_t$ from a distribution $P_t \in \cP$, where $\cP$ is the set of all probability distributions supported on $\cX$. The goal of the learner is to choose a sequence of distributions $\{P_t\}_{t=1}^T$ to minimize the expected regret
$
\sum_{t = 1}^T \Eover{\x\sim P_t}{f_t(\x)} - \inf_{\x \in \cX}\sum_{t=1}^Tf_t(\x).
$
A popular technique to minimize the expected regret is to consider a linearized problem in the space of probability distributions with losses $\Tilde{f}_t(P) = \Eover{\x\sim P}{f_t(\x)}$ and perform FTRL in this space. In such a technique, $P_t$ is computed as:
$
\argmin_{P \in \cP} \sum_{i=1}^{t-1} \Tilde{f}_i(P) +  R(P),
$
for some strongly convex regularizer $R(P).$ 
When $R(P)$ is the negative entropy of $P$, the algorithm is called entropic mirror descent or continuous exponential weights. 
This algorithm achieves $\order{T^{1/2}}$ expected regret for bounded loss functions $f_t$. Another technique to minimize expected regret is to rely on  FTPL~\citep{gonen2018learning, suggala2019online}. Here, the learner generates the random prediction $\x_t$ by first sampling a random perturbation $\sigma$ and then computing the perturbed best response, which is defined as $\argmin_{\x \in \cX} \sum_{i=1}^{t-1}f_i(\x) - \iprod{\sigma}{\x}$. In a recent work, \citet{suggala2019online} show that this algorithm achieves $\order{T^{1/2}}$ expected regret, whenever the sequence of loss functions are Lipschitz. We now briefly discuss the computational aspects of FTRL and FTPL. Each iteration of FTRL (with entropic regularizer) requires sampling from a non-logconcave distribution. In contrast, FTPL requires solving a nonconvex optimization problem to compute the perturbed best response. Of these, computing the perturbed best response seems significantly easier since standard algorithms such as gradient descent seem to be able to find approximate global optima reasonably fast, even for complicated tasks such as training deep neural networks.
\vspace{-0.05in}
\paragraph{Online Learning with Optimism.} When the sequence of loss functions are convex and predictable, \citet{rakhlin2012online, rakhlin2013optimization} study optimistic variants of FTRL which can exploit the predictability to obtain better regret bounds. Let $g_t$ be our guess of $\grad_t$ at the beginning of round $t$. Given $g_t$, we predict $\x_t$ in Optimistic FTRL (OFTRL) as
$
\argmin_{\x\in\cX}\iprod{\sum_{i=1}^{t-1}\grad_i + g_t}{\x} + R(\x).
$ 
Note that when $g_t=0$, OFTRL is equivalent to FTRL. \citep{rakhlin2012online, rakhlin2013optimization} show that the regret bounds of OFTRL only depend on $(g_t-\grad_t)$. Moreover, these works show that OFTRL provides faster convergence rates for solving smooth convex-concave games.
In contrast to FTRL, the optimistic variants of FTPL have been less well understood.
\citep{rakhlin2012online} studies OFTPL for linear loss functions. But they consider restrictive settings and their algorithms require the knowledge of sizes of deviations $(g_t-\nabla_t)$. \citep{suggala2019online} studies OFTPL for the more general nonconvex setting.  The algorithm predicts $\x_t$ as $\argmin_{\x \in \cX} \sum_{i=1}^{t-1}f_i(\x) +g_t(\x) - \iprod{\sigma}{\x}$, where $g_t$ is our guess of $f_t$. However, the regret bounds of~\citep{suggala2019online} are sub-optimal and weaker than the bounds we obtain in our work (see Theorem~\ref{thm:oftpl_noncvx_regret}). Moreover, \citep{suggala2019online} does not provide any consequences of their results to minimax games. We note that their sub-optimal regret bounds translate to sub-optimal rates of convergence for solving smooth minimax games.
\vspace{-0.05in}
\paragraph{Minimax Games.} Consider the following problem, which we refer to as minimax game: $\min_{\x \in \cX}\max_{\y \in \cY} f(\x,\y)$. In these games, we are often interested in finding a Nash Equilibrium (NE). A pair $(P,Q)$, where $P$ is a probability distribution over $\cX$ and $Q$ is a probability distribution over $\cY$, is called a NE if: $\sup_{\y\in\cY}\Eover{\x\sim P}{f(\x,\y)}\leq\Eover{\x\sim P, \y \sim Q}{f(\x,\y)} \leq \inf_{\x\in\cX}\Eover{\y\sim Q}{f(\x,\y)}.$
A standard technique for finding a NE of the game is to rely on no-regret algorithms~\citep{cesa2006prediction, hazan2016introduction}. Here, both $\x$ and $\y$ players play a repeated game against each other and use online learning algorithms to choose their actions. The average of the iterates generated via this repeated game can be shown to converge to a NE.
\vspace{-0.05in}
\paragraph{Projection Free Learning.} Projection free learning algorithms are attractive as they only involve solving linear optimization problems. Two broad classes of projection free techniques have been considered for online convex learning and minimax games, namely, Frank-Wolfe (FW) methods and FTPL based methods. \citet{garber2013playing} consider the problem of online learning when the action space $\cX$ is a polytope. They provide a FW method which achieves $\order{T^{1/2}}$ regret using $T$ calls to the linear optimization oracle. \citet{hazan2012projection} provide a FW technique which achieves $\order{T^{3/4}}$ regret for general online convex learning with Lipschitz losses and uses $T$ calls to the linear optimization oracle. In a recent work,~\citet{hazan2020projection} show that FTPL achieves $\order{T^{2/3}}$ regret for online convex learning with smooth losses, using $T$ calls to the linear optimization oracle. This translates to $\order{T^{-1/3}}$ rate of convergence for solving smooth convex-concave games. Note that, in contrast, our algorithm achieves $\order{T^{-1/2}}$ convergence rate in the same setting.  \citet{gidel2016frank} study FW methods for solving convex-concave games. When the constraint sets $\cX,\cY$ are \emph{strongly convex}, the authors show geometric convergence of their algorithms. In a recent work,~\citet{he2015semi} propose a FW technique for solving smooth convex-concave games which converges at a rate of $\order{T^{-1/2}}$ using $T$ calls to the linear optimization oracle. We note that our simple OFTPL based algorithm achieves these rates, with the added advantage of parallelizability. That being said, \citet{he2015semi} achieve dimension free convergence rates in the Euclidean setting, where the smoothness is measured w.r.t $\|\cdot\|_2$ norm. In contrast, the rates of convergence of our algorithm depend on the dimension. \vspace{-0.05in}
\paragraph{Notation.} $\|\cdot\|$ is a norm on some vector space, which is typically $\mathbb{R}^d$ in our work.  $\|\cdot\|_{*}$ is the dual norm of $\|\cdot\|$, which is defined as $\|\x\|_{*} = \sup\{\iprod{\u}{\x}: \u\in\mathbb{R}^d, \|\u\|\leq 1\}$. We use $\normcompf, \normcompb$ to denote norm compatibility constants of $\|\cdot\|$, which are defined as $\normcompf = \sup_{\x\neq 0} \|\x\|/\|\x\|_2,\  \normcompb = \sup_{\x\neq 0} \|\x\|_2/\|\x\|.$
We use the notation $f_{1:t}$ to denote $\sum_{i=1}^tf_i$. In some cases, when clear from context, we overload the notation $f_{1:t}$ and use it to denote the set $\{f_1,f_2\dots f_t\}$. For any convex function $f$, $\partial f(\x)$ is the set of all subgradients of $f$ at $\x$. For any function $f:\cX\times \cY \to \mathbb{R}$,  $f(\cdot,\y), f(\x,\cdot)$ denote the functions $\x\rightarrow f(\x,\y), \y\rightarrow f(\x,\y).$ For any function $f:\cX\to\mathbb{R}$ and any probability distribution $P$, we let $f(P)$ denote $\Eover{\x\sim P}{f(\x)}.$ Similarly, for any function $f:\cX\times\cY\to\mathbb{R}$ and any two distributions $P,Q$, we let $f(P,Q)$ denote $\Eover{\x\sim P,\y\sim Q}{f(\x,\y)}.$  For any set of distributions $\{P_j\}_{j=1}^m$, $\frac{1}{m}\sum_{j=1}^mP_j$ is the mixture distribution which gives equal weights to its components. We use $\text{Exp}(\eta)$ to denote the exponential distribution, whose CDF is given by $P(Z\leq s) =1-\exp(-s/\eta).$  
\section{Dual view of Perturbation as Regularization}
\label{sec:duality}
In this section, we present a key result which shows that when the sequence of loss functions are convex, every FTPL algorithm is an FTRL algorithm. Our analysis of OFTPL relies on this dual view to obtain tight regret bounds. This duality between FTPL and FTRL was originally studied by~\citet{hofbauer2002global}, where the authors show that any FTPL algorithm, with perturbation distribution admitting a strictly positive density on $\mathbb{R}^d$, is an FTRL algorithm w.r.t some convex regularizer. However, many popular perturbation distributions such as exponential and uniform distributions don't have a strictly positive density. In a recent work, \citet{abernethy2016perturbation} point out that the duality between FTPL and FTRL holds for very general perturbation distributions. However, the authors do not provide a formal theorem showing this result. Here, we provide a proposition formalizing the claim of~\citep{abernethy2016perturbation}.
\begin{proposition}
\label{prop:ftpl_ftrl_connection}
Consider the problem of online convex learning, where the sequence of loss functions $\{f_t\}_{t=1}^T$ encountered by the learner are convex. Consider the deterministic version of FTPL algorithm, where the learner predicts $\x_t$ as $\Eover{\sigma}{\argmin_{\x \in \cX}\iprod{\grad_{1:t-1}-\sigma}{\x}}$.
Suppose the perturbation distribution is absolutely continuous w.r.t the Lebesgue measure. Then there exists a convex regularizer $R:\mathbb{R}^d\to \mathbb{R}\cup\{\infty\}$, with domain $\dom{R}\subseteq \cX$, such that 
$
\x_t  = \argmin_{\x \in \cX} \iprod{\grad_{1:t-1}}{\x}+R(\x).
$
Moreover, \mbox{$-\grad_{1:t-1} \in \partial R(\x_t),$} and $\x_t = \partial R^{-1}\left(-\grad_{1:t-1}\right),$ where $\partial R^{-1}$ is the inverse of $\partial R$ in the sense of multivalued mappings.
\end{proposition}

\section{Online Learning with OFTPL}
\label{sec:onlinelearning}

\subsection{Online Convex Learning}
\begin{algorithm}[t]
\caption{Convex OFTPL}
\label{alg:oftpl_cvx}
\begin{algorithmic}[1]
  \State \textbf{Input:}  Perturbation Distribution $\pertdist,$ number of samples $m,$ number of iterations $T$
  \State Denote $\grad_0 = 0$
  \For{$t = 1 \dots T$}
  \State Let $g_t$ be the guess for $\grad_t$
  \For{$j=1\dots m$}
  \State Sample $\sigma_{t,j}\sim \pertdist$
  \State $\x_{t,j}\in \argmin_{\x \in \cX}\iprod{\grad_{0:t-1}+  g_t-\sigma_{t,j}}{\x}$
  \EndFor
  \State Play $\x_t=\frac{1}{m}\sum_{j=1}^m \x_{t,j}$
  \State Observe loss function $f_t$
  \EndFor
\end{algorithmic}
\end{algorithm}
In this section, we present the OFTPL algorithm for online convex learning and derive an upper bound on its regret. The algorithm we consider is similar to the OFTRL algorithm (see Algorithm~\ref{alg:oftpl_cvx}).
Let $g_t[f_1\dots f_{t-1}]$ be our guess for $\grad_t$ at the beginning of round $t$, with $g_1 = 0$.  To simplify the notation, in the sequel, we suppress the dependence of $g_t$ on $\{f_{i}\}_{i=1}^{t-1}$. Given $g_t$, we predict $\x_t$ in OFTPL as follows. We sample independent perturbations $\{\sigma_{t,j}\}_{j=1}^m$ from the perturbation distribution $\pertdist$ and compute $\x_t$ as $m^{-1}\sum_{j=1}^m\x_{t,j}$, where $\x_{t,j}$ is a minimizer of the following linear optimization problem $$\x_{t,j} \in \argmin_{\x \in \cX} \iprod{\grad_{1:t-1} + g_t - \sigma_{t,j}}{\x}.$$

We now present our main theorem which bounds the regret of OFTPL. A key quantity the regret depends on is the \emph{stability} of predictions of the deterministic version of OFTPL. Intuitively, an algorithm is stable if its predictions in two consecutive iterations differ by a small quantity. To capture this notion, we first define function $\gphiNA:\mathbb{R}^d \to \mathbb{R}^d$ as:
$
    \gphi{g} = \Eover{\sigma}{\argmin_{\x \in \cX}\left\langle g-\sigma, \x\right\rangle}.
$
Observe that $\gphi{\nabla_{1:t-1} + g_t}$ is the prediction of the deterministic version of OFTPL. We say the predictions of OFTPL are stable, if $\gphiNA$ is a Lipschitz function.
\begin{defn}[Stability]
The predictions of OFTPL are said to be $\beta$-stable w.r.t some norm $\|\cdot\|$, if 
\[
\forall g_1,g_2\in\mathbb{R}^d \quad \|\gphi{g_1} - \gphi{g_2}\|_{*} \leq \beta \|g_1-g_2\|.
\]
\end{defn}
\begin{theorem}
\label{thm:oftpl_regret}
Suppose the perturbation distribution $\pertdist$ is absolutely continuous w.r.t Lebesgue measure. 
Let $D$ be the diameter of $\cX$ w.r.t $\|\cdot\|$, which is defined as $D= \sup_{\x_1,\x_2\in\cX} \|\x_1-\x_2\|.$ 
Let \mbox{$\eta=\Eover{\sigma}{\|\sigma\|_*},$} and suppose the predictions of OFTPL are $\stability \eta^{-1}$-stable w.r.t $\|\cdot\|_*$, where $\stability $ is a constant that depends on the set $\mathcal{X}.$
Finally, suppose the sequence of loss functions $\{f_t\}_{t=1}^T$ are Holder smooth and satisfy
\[
\forall \x_1,\x_2\in \cX\quad \|\grad f_t(\x_1)-\grad f_t(\x_2)\|_* \leq L\|\x_1-\x_2\|^{\alpha},
\]
for some constant $\alpha \in [0,1]$.
Then the expected regret of Algorithm~\ref{alg:oftpl_cvx} satisfies
\begin{align*}
    \sup_{\x\in\cX} \E{\sum_{t=1}^Tf_t(\x_t) -  f_t(\x)} &\leq  \eta D + \sum_{t=1}^T \frac{\stability }{2\eta}\E{\|\grad_t-g_{t}\|_{*}^2}  - \sum_{t=1}^T\frac{\eta}{2\stability } \E{\|\x_t^{\infty}-\Tilde{\x}_{t-1}^{\infty}\|^2} \\
    &\quad + LT\left(\frac{\normcompf \normcompb D}{\sqrt{m}}\right)^{1+\alpha}.
\end{align*}
where $\x_t^{\infty} = \E{\x_t|g_t,f_{1:t-1}, \x_{1:t-1}}$ and $\Tilde{\x}_{t-1}^{\infty} = \E{\Tilde{\x}_{t-1}|f_{1:t-1}, \x_{1:t-1}}$ and $\Tilde{\x}_{t-1}$ denotes the prediction in the $t^{th}$ iteration of Algorithm~\ref{alg:oftpl_cvx}, if guess $g_{t}=0$ was used. Here, $\normcompf, \normcompb$ denote the norm compatibility constants of $\|\cdot\|.$
\end{theorem}
Regret bounds that hold with high probability can be found in Appendix~\ref{sec:hp_bounds}. The above Theorem shows that the regret of OFTPL only depends on $\|\grad_t-g_{t}\|_{*}$, which quantifies the accuracy of our guess $g_t$. In contrast, the regret of FTPL depends on $\|\grad_t\|_{*}$~\citep{hazan2016introduction}. This shows that for predictable sequences, with an appropriate choice of $g_t$, OFTPL can achieve better regret guarantees than FTPL.  
As we demonstrate in Section~\ref{sec:games}, this helps us design faster algorithms for solving minimax games. 

Note that the above result is very general and holds for any absolutely continuous perturbation distribution.
The key challenge in instantiating this result for any particular perturbation distribution is in showing the stability of predictions. Several past works have studied the stability of FTPL for various perturbation distributions such as uniform, exponential, Gumbel distributions~\citep{kalai2005efficient,hazan2016introduction, hazan2020projection}. Consequently, the above result can be used to derive tight regret bounds for all these perturbation distributions.
As one particular instantiation of Theorem~\ref{thm:oftpl_regret}, we consider the special case of $g_t = 0$ and derive regret bounds for FTPL, when the perturbation distribution is the uniform distribution over a ball centered at the origin.
\begin{corollary}[FTPL]
\label{cor:ftpl_cvx_gaussian}
Suppose the perturbation distribution is equal to the uniform distribution over $\{\x:\|\x\|_2 \leq (1+d^{-1})\eta\}.$  Let $D$ be the diameter of $\cX$ w.r.t $\|\cdot\|_2$. Then $\Eover{\sigma}{\|\sigma\|_2} = \eta$, and the predictions of OFTPL are $dD\eta^{-1}$-stable w.r.t $\|\cdot\|_2$.
Suppose, the sequence of loss functions $\{f_t\}_{t=1}^T$ are $G$-Lipschitz and satisfy $\sup_{\x \in \mathcal{X}} \|\grad f_t(\x)\|_2 \leq G$. Moreover, suppose $f_t$ satisfies the Holder smooth condition in Theorem~\ref{thm:oftpl_regret} w.r.t $\|\cdot\|_2$ norm. Then the expected regret of Algorithm~\ref{alg:oftpl_cvx}, with guess $g_t = 0$, satisfies
\begin{align*}
    \sup_{\x\in\cX} \E{\sum_{t=1}^Tf_t(\x_t) -  f_t(\x)} &\leq  \eta D + \frac{dDG^2 T}{2\eta} + LT\left(\frac{D}{\sqrt{m}}\right)^{1+\alpha}.
\end{align*}
\end{corollary}
This recovers the regret bounds of FTPL for general convex loss functions, derived by~\cite{hazan2020projection}.

\subsection{Online Nonconvex Learning}
\label{sec:online_noncvx}
\begin{algorithm}[t]
\caption{Nonconvex OFTPL}
\label{alg:oftpl_noncvx}
\begin{algorithmic}[1]

  \State \textbf{Input:}   Perturbation Distribution $\pertdist,$ number of samples $m$, number of iterations $T$
    \State Denote $f_0=0$
  \For{$t = 1 \dots T$}
  \State Let $g_t$ be the guess for $f_t$
  \For{$j=1\dots m$}
  \State Sample $\sigma_{t,j}\sim \pertdist$
  \State  $\x_{t,j} \in \argmin_{\x\in\cX}f_{0:t-1}(\x)+g_t(\x)-\sigma_{t,j}(\x)$
  \EndFor
  \State Let $P_t $ be the empirical distribution over $\{\x_{t,1}, \x_{t,2} \dots \x_{t,m}\}$
  \State Play $\x_t$, a random sample generated from $P_t$
  \State Observe loss function $f_t$
  \EndFor
\end{algorithmic}
\end{algorithm}
We now study OFTPL in the nonconvex setting. In this setting,  we assume the sequence of loss functions belong to some function class $\cF$ containing real-valued measurable functions on $\cX$. Some popular choices for  $\cF$ include the set of Lipschitz functions, the set of bounded functions. 
The OFTPL algorithm in this setting is described in Algorithm~\ref{alg:oftpl_noncvx}. Similar to the convex case, we first sample random perturbation functions $\{\sigma_{t,j}\}_{j=1}^m$ from some distribution $\pertdist$. Some examples of perturbation functions that have been considered in the past include $\sigma_{t,j}(\x) = \iprod{\bar{\sigma}_{t,j}}{\x},$ for some random vector $\bar{\sigma}_{t,j}$ sampled from exponential or uniform distributions~\citep{gonen2018learning, suggala2019online}. Another popular choice for $\sigma_{t,j}$ is the Gumbel process, which results in the continuous exponential weights algorithm~\citep{maddison2014sampling}. Letting, $g_t$ be our guess of loss function $f_t$ at the beginning of round $t$, the learner first computes $\x_{t,j}$ as  
$\argmin_{\x \in \cX}\sum_{i = 1}^{t-1}f_i(\x)+g_t(\x)-\sigma_{t,j}(\x).$ We assume access to an optimization oracle which computes a minimizer of this problem.  We often refer to this oracle as the \emph{perturbed best response} oracle.
Let $P_t$ denote the empirical distribution of $\{\x_{t,j}\}_{j=1}^m$. The learner then plays an $\x_t$ which is  sampled from $P_t$. Algorithm~\ref{alg:oftpl_noncvx} describes this procedure. We note that for the online learning problem, $m=1$ suffices, as the expected loss suffered by the learner in each round is independent of $m$; that is $\E{f_t(\x_t)} = \E{f_t(\x_{t,1})}$. However, the choice of $m$ affects the rate of convergence when Algorithm~\ref{alg:oftpl_noncvx} is used for solving nonconvex nonconcave minimax games.

Before we present the regret bounds, we introduce the \emph{dual space} associated with $\cF$. Let $\|\cdot\|_{\cF}$ be a seminorm associated with $\cF$. For example, when $\cF$ is the set of Lipschitz functions, $\|\cdot\|_{\cF}$ is the Lipschitz seminorm. Various choices of $(\cF,\|\cdot\|_{\cF})$ induce various distance metrics on $\cP$, the set of all probability distributions on $\cX$. We let $\gammaF$ denote the Integral Probability Metric (IPM) induced by $(\cF,\|\cdot\|_{\cF})$, which is defined as
\[
\gammaF(P,Q) = \sup_{f\in\cF,\|f\|_{\cF} \leq 1}\Big|\Eover{\x\sim P}{f(\x)} - \Eover{\x\sim Q}{f(\x)}\Big|.
\]
We often refer to $(\cP, \gammaF)$ as the dual space of $(\cF,\|\cdot\|_{\cF})$. When $\cF$ is the set of Lipschitz functions and when $\|\cdot\|_{\cF}$ is the Lipschitz seminorm, $\gammaF$ is the Wasserstein distance. Table~\ref{tab:ipm} in Appendix~\ref{sec:primal_dual_spaces} presents examples of $\gammaF$ induced by some popular function spaces. Similar to the convex case, the regret bounds in the nonconvex setting depend on the stability of predictions of OFTPL. 
\begin{defn}[Stability]
Suppose the perturbation function $\sigma(\x)$ is sampled from $\pertdist$. For any $f\in\cF$, define random variable $\x_f(\sigma)$ as
 $\argmin_{\x \in \cX} f(\x)-\sigma(\x).$ Let $\gphi{f}$ denote the distribution of $\x_f(\sigma)$.
The predictions of OFTPL are said to be $\beta$-stable w.r.t $\|\cdot\|_{\cF}$ if 
\[
\forall f,g\in\cF \quad \gammaF(\gphi{f}, \gphi{g}) \leq \beta \|f-g\|_{\cF}.
\]
\end{defn}
\begin{theorem}
\label{thm:oftpl_noncvx_regret}
Suppose the sequence of loss functions $\{f_t\}_{t=1}^T$ belong to $(\cF, \|\cdot\|_{\cF})$. Suppose the perturbation distribution $\pertdist$ is such that $\argmin_{\x\in\cX} f(\x) - \sigma(\x)$ has a unique minimizer with probability one, for any $f\in\cF$.  Let $\cP$ be the set of probability distributions over $\cX$. 
Define the diameter of $\cP$ as $D= \sup_{P_1,P_2\in\cP} \gammaF(P_1,P_2).$ Let $\eta=\E{\|\sigma\|_{\cF}}$.  
Suppose the predictions of OFTPL are $\stability \eta^{-1}$-stable w.r.t $\|\cdot\|_{\cF}$, for some constant $\stability $ that depends on $\cX$.  
Then the expected regret of Algorithm~\ref{alg:oftpl_noncvx} satisfies
\begin{align*}
    \sup_{\x\in\cX} \E{\sum_{t=1}^Tf_t(\x_t)-f_t(\x)} &\leq  \eta D + \sum_{t=1}^T \frac{\stability }{2\eta}\E{\|f_t-g_{t}\|_{\cF}^2}  -\sum_{t=1}^T \frac{\eta}{2\stability }\E{\gammaF(P_t^{\infty},\Tilde{P}_{t-1}^{\infty})^2},
\end{align*}
where $P_t^{\infty} = \E{P_t|g_t, f_{1:t-1}, P_{1:t-1}}, $ $\Tilde{P}_{t}^{\infty} = \E{\Tilde{P}_{t-1}|f_{1:t-1}, P_{1:t-1}}$ and $\Tilde{P}_{t-1}$ is the empirical distribution computed in the $t^{th}$ iteration of Algorithm~\ref{alg:oftpl_noncvx}, if guess $g_t = 0$ was used.
\end{theorem}
We note that, unlike the convex case, there are no known analogs of Fenchel duality for infinite dimensional function spaces. As a result, more careful analysis is needed to obtain the above regret bounds. Our analysis mimics the arguments made in the convex case, albeit without explicitly relying on duality theory.
As in the convex case, the key challenge in instantiating the above result for any particular perturbation distribution is in showing the stability of predictions. In a recent work,~\citep{suggala2019online} consider linear perturbation functions $\sigma(\x) = \iprod{\bar{\sigma}}{\x},$ for $\bar{\sigma}$ sampled from exponential distribution, and show stability of FTPL. We now instantiate the above Theorem for this setting. 
\begin{corollary}
\label{cor:ftpl_noncvx_exp}
Consider the setting of Theorem~\ref{thm:oftpl_noncvx_regret}. Let $\cF$ be the set of Lipschitz functions and $\|\cdot\|_{\cF}$ be the Lipschitz seminorm, which is defined as $\|f\|_{\cF}=\sup_{\x\neq \y \text{ in }\cX} |f(\x)-f(\y)|/\|\x-\y\|_1$.  
Suppose the perturbation function is such that $\sigma(\x) = \iprod{\bar{\sigma}}{\x}$, where $\bar{\sigma} \in \mathbb{R}^d$ is a random vector whose entries are sampled independently from $\text{Exp}(\eta)$.  Then $\Eover{\sigma}{\|\sigma\|_{\cF}} = \eta\log{d}$, and the predictions of OFTPL are $\order{d^2D\eta^{-1}}$-stable w.r.t $\|\cdot\|_{\cF}$.  Moreover, the expected regret of Algorithm~\ref{alg:oftpl_noncvx} is upper bounded by
\mbox{$\order{\eta D\log{d} + \sum_{t=1}^T \frac{d^2D }{\eta}\E{\|f_t-g_{t}\|_{\cF}^2}  -\sum_{t=1}^T \frac{\eta}{d^2D }\E{\gammaF(P_t^{\infty},\Tilde{P}_{t-1}^{\infty})^2}}.$}
\end{corollary}
We note that the above regret bounds are tighter than the regret bounds of \citep{suggala2019online}, where the authors show that the regret of OFTPL is bounded by $\order{\eta D\log{d} + \sum_{t=1}^T \frac{d^2D }{\eta}\E{\|f_t-g_{t}\|_{\cF}^2}}$. These tigher bounds help us design faster algorithms for solving minimax games in the nonconvex setting.
 
\vspace{-1mm}
\section{Minimax Games}
\label{sec:games}

We now consider the problem of solving minimax games of the following form
\begin{equation}
\label{eqn:minimax_game}
    \min_{\x \in \cX} \max_{\y \in \cY} f(\x,\y).
\end{equation}
Nash equilibria of such games can be computed by playing two online learning algorithms against each other~\citep{cesa2006prediction, hazan2016introduction}. In this work, we study the algorithm where both the players employ OFTPL to decide their actions in each round. For convex-concave games, both the players use the OFTPL algorithm described in Algorithm~\ref{alg:oftpl_cvx} (see Algorithm~\ref{alg:oftpl_cvx_games} in Appendix~\ref{sec:cvx-games}).   
The following theorem derives the rate of convergence of this algorithm to a Nash equilibirum (NE).   
\begin{theorem}
\label{thm:oftpl_cvx_smooth_games_uniform}
Consider the minimax game in Equation~\eqref{eqn:minimax_game}. Suppose both the domains $\cX,\cY$ are compact subsets of $\mathbb{R}^d$, with diameter \mbox{$D = \max\{\sup_{\x_1,\x_2\in\cX} \|\x_1-\x_2\|_2, \sup_{\y_1,\y_2\in\cY} \|\y_1-\y_2\|_2\}$.} Suppose $f$ is convex in $\x$, concave in $\y$ and is smooth w.r.t $\|\cdot\|_2$
\begin{align*}
    \|\grad_\x f(\x,\y)-\grad_\x f(\x',\y')\|_{2}+ \|\grad_\y f(\x,\y)-\grad_\y f(\x',\y')\|_{2} \leq L\|\x-\x'\|_2 + L\|\y-\y'\|_2.
\end{align*}
Suppose Algorithm~\ref{alg:oftpl_cvx_games} is used to solve the minimax game. Suppose the perturbation distributions used by both the players are the same and equal to the uniform distribution over $\{\x:\|\x\|_2 \leq (1+d^{-1})\eta\}.$  Suppose the guesses used by $\x,\y$ players in the $t^{th}$ iteration are $\grad_{\x}f(\Tilde{\x}_{t-1},\Tilde{\y}_{t-1}), \grad_{\y}f(\Tilde{\x}_{t-1},\Tilde{\y}_{t-1})$, where  $\Tilde{\x}_{t-1},\Tilde{\y}_{t-1}$ denote the predictions of $\x,\y$ players in the $t^{th}$ iteration, if guess $g_t = 0$ was used. If Algorithm~\ref{alg:oftpl_cvx_games} is run with $\eta = 6dD(L+1), m = T$, then the iterates $\{(\x_t,\y_t)\}_{t=1}^T$ satisfy
\begin{align*}
 \sup_{\x\in\cX,\y\in\cY}\E{f\left(\frac{1}{T}\sum_{t=1}^T\x_t,\y\right) - f\left(\x,\frac{1}{T}\sum_{t=1}^T\y_t\right)}=  \order{\frac{dD^2(L+1)}{T}}.
\end{align*}
\end{theorem}
Rates of convergence which hold with high probability can be found in Appendix~\ref{sec:hp_bounds}. 
We note that Theorem~\ref{thm:oftpl_cvx_smooth_games_uniform} can be extended to more general noise distributions and settings where gradients of $f$ are Holder smooth w.r.t non-Euclidean norms, and $\cX,\cY$ lie in spaces of different dimensions (see Theorem~\ref{thm:oftpl_cvx_smooth_games} in Appendix). 
The above result shows that for smooth convex-concave games, Algorithm~\ref{alg:oftpl_cvx_games} converges to a NE at $\order{T^{-1}}$ rate using $T^2$ calls to the linear optimization oracle. Moreover, the algorithm runs in $\order{T}$ iterations, with each iteration making  $\order{T}$ parallel calls to the optimization oracle. We believe the dimension dependence in the rates can be removed by appropriately choosing the perturbation distributions based on domains $\cX, \cY$ (see Appendix~\ref{sec:pert_dist_choice}). 

We now consider the more general nonconvex-nonconcave games. In this case, both the players use the nonconvex OFTPL algorithm described in Algorithm~\ref{alg:oftpl_noncvx} to choose their actions. Instead of generating a single sample from the empirical distribution $P_t$ computed in $t^{th}$ iteration of Algorithm~\ref{alg:oftpl_noncvx}, the players now play the entire distribution $P_t$ (see Algorithm~\ref{alg:oftpl_noncvx_games} in Appendix~\ref{sec:ncvx-games}). Letting $\{P_t\}_{t=1}^T, \{Q_t\}_{t=1}^T$, be the sequence of iterates generated by the $\x$ and $\y$ players, the following theorem shows that $\left(\frac{1}{T}\sum_{t=1}^TP_t,\frac{1}{T}\sum_{t=1}^TQ_t\right)$ converges to a NE.
\begin{theorem}
\label{thm:oftpl_noncvx_smooth_games_exp}
Consider the minimax game in Equation~\eqref{eqn:minimax_game}. Suppose the domains $\cX,\cY$ are compact subsets of $\mathbb{R}^d$ with diameter $D = \max\{\sup_{\x_1,\x_2\in\cX} \|\x_1-\x_2\|_1, \sup_{\y_1,\y_2\in\cY} \|\y_1-\y_2\|_1\}$.  Suppose $f$ is Lipschitz w.r.t $\|\cdot\|_1$ and satisfies 
\begin{align*}
\max\left\lbrace\sup_{\x\in\cX, \y\in\cY} \|\grad_{\x}f(\x,\y)\|_{\infty}, \sup_{\x\in\cX,\y\in\cY}\|\grad_{\y}f(\x,\y)\|_{\infty}\right\rbrace\leq G.
\end{align*}
Moreover, suppose $f$ satisfies the following smoothness property
\begin{align*}
    \|\grad_\x f(\x,\y)-\grad_\x f(\x',\y')\|_{\infty} + \|\grad_\y f(\x,\y)-\grad_\y f(\x',\y')\|_{\infty} \leq L\|\x-\x'\|_1 + L\|\y-\y'\|_1.
\end{align*}
Suppose both $\x$ and $\y$ players use Algorithm~\ref{alg:oftpl_noncvx_games} to solve the game with linear perturbation functions $\sigma(\z)=\iprod{\bar{\sigma}}{\z}$, where $\bar{\sigma} \in \mathbb{R}^d$ is such that each of its entries is sampled independently from $\text{Exp}(\eta)$.  Suppose the guesses used by $\x$ and $\y$ players in the $t^{th}$ iteration are $f(\cdot,\Tilde{Q}_{t-1}), f(\Tilde{P}_{t-1},\cdot)$, where $\Tilde{P}_{t-1},\Tilde{Q}_{t-1}$ denote the predictions of $\x,\y$ players in the $t^{th}$ iteration, if guess $g_t = 0$ was used. If Algorithm~\ref{alg:oftpl_noncvx_games} is run with $\eta = 10d^2D(L+1), m = T$, then the iterates $\{(P_t,Q_t)\}_{t=1}^T$ satisfy
\begin{align*}
 \sup_{\x\in\cX,\y\in\cY}\E{f\left(\frac{1}{T}\sum_{t=1}^TP_t,\y\right) - f\left(\x,\frac{1}{T}\sum_{t=1}^TQ_t\right)}& = \order{\frac{d^2D^2(L+1)\log{d}}{T}}\\
 &\quad + \order{\min\left\lbrace D^2L, \frac{d^2G^2\log{T}}{LT}\right\rbrace}.
\end{align*}
\end{theorem}
More general versions of the Theorem, which consider other function classes and general perturbation distributions, can be found in Appendix~\ref{sec:ncvx-games}. The above result shows that Algorithm~\ref{alg:oftpl_noncvx_games} converges to a NE at $\torder{T^{-1}}$ rate using $T^2$ calls to the perturbed best response oracle. This matches the rates of convergence of FTPL~\citep{suggala2019online}. However, the key advantage of our algorithm is that it is highly parallelizable and runs in $\order{T}$ iterations, in contrast to FTPL, which runs in $\order{T^2}$ iterations.

\vspace{-1mm}
\section{Conclusion}
\label{sec:conclusion}
\vspace{-1mm}
We studied an optimistic variant of FTPL which achieves better regret guarantees when the sequence of loss functions is predictable. As one specific application of our algorithm, we considered the problem of solving minimax games. For solving convex-concave games, our algorithm requires access to a linear optimization oracle and for nonconvex-nonconcave games our algorithm requires access to a more powerful perturbed best response oracle. In both these settings, our algorithm achieves $\order{T^{-1/2}}$ convergence rates using $T$ calls to the oracles. Moreover, our algorithm runs in $\order{T^{1/2}}$ iterations, with  each iteration making $\order{T^{1/2}}$ parallel calls to the optimization oracle. We believe our improved algorithms for solving minimax games are useful in a number of modern machine learning applications such as training of GANs, adversarial training, which involve solving nonconvex-nonconcave minimax games and often deal with huge datasets.

%
 
\bibliography{local}
\bibliographystyle{unsrtnat}
\clearpage
\appendix
\section{Dual view of Perturbations as Regularization}
\subsection{Proof of Theorem~\ref{prop:ftpl_ftrl_connection}}
We first define a convex function $\Psi:\mathbb{R}^d \to \mathbb{R}$ as $$\Psi(f)=\Eover{\sigma}{\sup_{\x \in \cX}\iprod{f+\sigma}{\x}} = \Eover{\sigma}{\sup_{\x \in \cX}\iprod{f+\sigma}{\x}},$$
where perturbation $\sigma$ follows probability distribution $\pertdist$ which is absolutely continuous w.r.t the Lebesgue measure. 
For our choice of $\pertdist$, we now show that $\Psi$ is differentiable. 
Consider the function $\psi(g) = \sup_{\x\in\cX}\iprod{g}{\x}$. Since $\psi(g)$ is a proper convex function, we know that it is differentiable almost everywhere, except on a set of Lebesgue measure $0$~\citep[see Theorem 25.5 of][]{rockafellar1970convex}. 
Moreover, it is easy to verify that $\argmax_{\x\in\cX}\iprod{g}{\x} \in \partial \psi(g).$ 
These two observations, together with the fact that $\pertdist$ is absolutely continuous, show that the $\sup$ expression inside the expectation of $\Psi$ has a unique maximizer with probability one.

Since the sup expression inside the expectation has a unique maximizer with probability $1$, we can swap the expectation and gradient to obtain~\citep[see Proposition 2.2 of][]{bertsekas1973stochastic} 
\begin{equation}
\label{eqn:phi_gradient}
    \grad \Psi(f) = \Eover{\sigma}{\argmax_{\x \in \cX}\iprod{f+\sigma}{\x}}.
\end{equation}
Note that $\grad \Psi$ is related to the prediction of deterministic version of FTPL. Specifically, $\grad\Psi(-\grad_{1:t-1})$ is the prediction of deterministic FTPL in the $t^{th}$ iteration. We now show that  $\grad \Psi(f) = \argmin_{\x\in\cX} \iprod{-f}{\x} + R(\x)$, for some convex function $R$. 

Since all differentiable functions are closed, $\Psi(f)$ is a proper, closed and differentiable convex function over $\mathbb{R}^d$. Let $R(\x)$ denote the Fenchel conjugate of $\Psi(f)$
\[
R(\x) = \sup_{f\in \dom{\Phi}} \iprod{\x}{f} - \Psi(f),
\]
where $\dom{\Psi}$ denotes the domain of $\Psi$.
Following Theorem~\ref{thm:fenchel_prop1} (see Appendix~\ref{sec:fenchel_conjugate}), $\Psi(f)$ is the Fenchel conjugate of $R(\x)$
\begin{align*}
    \Psi(f) = \sup_{\x \in \dom{R}}\iprod{f}{\x} - R(\x).
\end{align*}
Furthermore, from Theorem~\ref{thm:fenchel_prop3} we have
\[
\grad\Psi(f) = \argmax_{\x\in \dom{R}}\iprod{f}{\x} - R(\x).
\]
We now show that the domain of $R$ is a subset of $\cX$. This, together with the previous two equations, would then immediately imply
\begin{align}
    \Psi(f) = \sup_{\x\in \cX}\iprod{f}{\x} - R(\x),\\
    \label{eqn:ftpl_ftrl_connection_proof}
    \grad\Psi(f) = \argmax_{\x\in \cX}\iprod{f}{\x} - R(\x).
\end{align}
From Theorem~\ref{thm:fenchel_prop2}, we know that the domain of $R$ satisfies
\[
\text{ri}(\dom{R}) \subseteq \text{range} \grad \Psi \subseteq \dom{R},
\]
where $\text{ri}(A)$ denotes the relative interior of a set $A$.
Moreover, from the definition of $\grad \Psi(f)$ in Equation~\eqref{eqn:phi_gradient}, we have $\text{range} \grad \Psi \subseteq \cX$. Combining these two properties, we can show that one of the following statements is true
\begin{align*}
    \text{ri}(\dom{R}) \subseteq \text{range} \grad \Psi \subseteq \cX \subseteq \dom{R},\\
    \text{ri}(\dom{R}) \subseteq \text{range} \grad \Psi \subseteq \dom{R} \subseteq \cX.
\end{align*}
Suppose the first statement is true. Since $\cX$ is a compact set, it is easy to see that \mbox{$\cX = \dom{R}$.} If the second statement is true, then $\dom{R} \subseteq \cX$. Together, these two statements imply $\dom{R} \subseteq \cX$.  
\paragraph{Connecting back to FTPL.} We now connect the above results to FTPL. From Equation~\eqref{eqn:phi_gradient}, we know that the prediction at iteration $t$ of deterministic FTPL is equal to $\grad \Psi(-\grad_{1:t-1}).$ From Equation~\eqref{eqn:ftpl_ftrl_connection_proof}, $\grad \Psi(-\grad_{1:t-1})$ is defined as
\[
\x_t = \grad\Psi(-\grad_{1:t-1}) = \argmax_{\x\in \cX}\iprod{-\grad_{1:t-1}}{\x} - R(\x).
\]
This shows that 
\[
\x_t = \argmin_{\x\in \cX}\iprod{\grad_{1:t-1}}{\x} + R(\x).
\]
So the prediction of FTPL can also be obtained using FTRL for some convex regularizer $R(\x)$. Finally,  to show that 
$-\grad_{1:t-1} \in \partial R(\x_t),  \x_t = \partial R^{-1}\left(-\grad_{1:t-1}\right),$
we rely on Theorem~\ref{thm:fenchel_prop4}. Since $\x_t = \grad\Psi(-\grad_{1:t-1})$, from Theorem~\ref{thm:fenchel_prop4}, we have
\[
-\grad_{1:t-1} \in \partial R(\x_t),\quad \x_t = \grad\Psi(-\grad_{1:t-1})= \partial R^{-1}\left(-\grad_{1:t-1}\right),
\]
where $\partial R^{-1}$ is the inverse of $\partial R$ in the sense of multivalued mappings. Note that, even though $\partial R$ can be a multivalued mapping, its inverse $\partial R^{-1} = \grad \Psi$ is a singlevalued mapping (this follows form differentiability of $\Psi$). This finishes the proof of the Theorem.
\section{Online Convex Learning}
\subsection{Proof of Theorem~\ref{thm:oftpl_regret}}
Before presenting the proof of the Theorem, we introduce some notation.
\subsubsection{Notation}
 We define functions $\Phi:\mathbb{R}^d\to \mathbb{R}$, $R:\mathbb{R}^d\to\mathbb{R}$ as follows
\begin{align*}
    &\Phi(f) = \Eover{\sigma}{\inf_{\x\in\cX}\iprod{f-\sigma}{\x}},\quad R(\x) = \sup_{f\in \mathbb{R}^d} \iprod{f}{\x} + \Phi(-f).
\end{align*}
Note that $\Phi$ is related to the function $\Psi$ defined in the proof of Proposition~\ref{prop:ftpl_ftrl_connection}. To be precise, $\Psi(f) = -\Phi(-f)$. Moreover, $R(\x)$ is the Fenchel conjugate of $\Psi$. 
 For our choice of perturbation distribution, $\Psi$ is differentiable (see proof of Proposition~\ref{prop:ftpl_ftrl_connection}). This implies $\Phi$ is also differentiable with gradient $\gphiNA$ defined as
\begin{align*}
    \gphi{f} = \Eover{\sigma}{\argmin_{\x\in\cX}\iprod{f-\sigma}{\x}}.
\end{align*}
Note that $\grad\Phi$ is the prediction of deterministic version of FTPL. In Proposition~\ref{prop:ftpl_ftrl_connection} we showed that 
\[
\gphi{f} =  \argmin_{\x\in \cX}\iprod{f}{\x} + R(\x).
\]
\subsubsection{Main Argument}
Since $\x_t^{\infty}$ is the prediction of deterministic version of FTPL, following FTPL-FTRL duality proved in Proposition~\ref{prop:ftpl_ftrl_connection}, $\x_t^{\infty}$ can equivalently be written as
\[
\x_t^{\infty} = \gphi{\grad_{1:t-1} + g_t} = \argmin_{\x\in\cX} \iprod{\grad_{1:t-1} + g_t}{\x} + R(\x).
\]
Similarly, $\Tilde{\x}_t^{\infty}$ can be written as
\[
\Tilde{\x}_t^{\infty}=\gphi{\grad_{1:t}} = \argmin_{\x\in\cX} \iprod{\grad_{1:t}}{\x} + R(\x).
\]
We use the notation $\grad_{1:0}=0$. So $\Tilde{\x}_0^{\infty},\x_1^{\infty}$ are equal to $\argmin_{\x\in\cX}  R(\x).$ From the first order optimality conditions, we have $$ -\grad_{1:t-1} - g_t \in \partial R\left(\x_t^{\infty}\right),\quad -\grad_{1:t} \in \partial R\left(\Tilde{\x}_t^{\infty}\right).$$
Define functions $B(\cdot,\x_t^{\infty}), B(\cdot,\Tilde{\x}_t^{\infty})$ for any $t\in[T]$ as
\begin{align*}
B(\x,\x_t^{\infty}) &= R(\x) - R(\x_t^{\infty}) + \iprod{\grad_{1:t-1} + g_t}{\x-\x_t^{\infty}},\\
B(\x,\Tilde{\x}_t^{\infty}) &= R(\x) - R(\Tilde{\x}_t^{\infty}) + \iprod{\grad_{1:t}}{\x-\Tilde{\x}_t^{\infty}}.
\end{align*}
From the stability of predictions of OFTPL we know that: 
$
\|\gphi{g_1} - \gphi{g_2}\| \leq C\eta^{-1}\|g_1-g_2\|_{*}.
$
Following our connection between $\Psi,\Phi$, this implies
$
\|\grad\Psi(g_1) - \grad\Psi(g_2)\| \leq C\eta^{-1}\|g_1-g_2\|_{*}.
$
This implies the following smoothness condition on $\Psi$~\citep[see Lemma 15 of][]{shalev2007thesis}
\[
\Psi(g_2) \leq \Psi(g_1) + \iprod{\grad\Psi(g_1)}{g_2-g_1} + \frac{C\eta^{-1}}{2}\|g_1-g_2\|_{*}^2.
\]
Since $\Psi$ is $C\eta^{-1}$-smooth w.r.t $\|\cdot\|_{*}$, following duality between strong convexity and strong smoothness properties (see Theorem~\ref{thm:fenchel_strong_convex_weak}), we can infer that $R$ is $\stability ^{-1}\eta$- strongly convex w.r.t $\|\cdot\|$ norm and satisfies $$B(\x,\x_t^{\infty}) \geq \frac{\eta}{2\stability }\|\x-\x_t^{\infty}\|^2, \quad B(\x,\Tilde{\x}_t^{\infty}) \geq \frac{\eta}{2\stability }\|\x-\Tilde{\x}_t^{\infty}\|^2.$$
We now go ahead and bound the regret of the learner. For any $\x\in \cX$, we have
\begin{align*}
f_t(\x_t) - f_t(\x) \stackrel{(a)}{\leq} \iprod{\x_t-\x}{\grad_t} & = \iprod{\x_t-\x_t^{\infty}}{\grad_t} + \iprod{\x_t^{\infty} - \x}{\grad_t}\\
&= \iprod{\x_t-\x_t^{\infty}}{\grad_t} + \iprod{\x_t^{\infty}-\Tilde{\x}_t^{\infty}}{\grad_t-g_{t}} + \iprod{\x_t^{\infty}-\Tilde{\x}_t^{\infty}}{g_{t}} \\
&\quad +\iprod{\Tilde{\x}_t^{\infty}-\x}{\grad_t}\\
& \leq \iprod{\x_t-\x_t^{\infty}}{\grad_t} + \|\x_t^{\infty}-\Tilde{\x}_t^{\infty}\|\|\grad_t-g_{t}\|_{*} + \iprod{\x_t^{\infty}-\Tilde{\x}_t^{\infty}}{g_{t}} \\
&\quad +\iprod{\Tilde{\x}_t^{\infty}-\x}{\grad_t},
\end{align*}
where $(a)$ follows from convexity of $f$.
Next, a simple calculation shows that
\begin{align*}
    \iprod{\x_t^{\infty}-\Tilde{\x}_t^{\infty}}{g_{t}} &= B(\Tilde{\x}_t^{\infty},\Tilde{\x}_{t-1}^{\infty}) - B(\Tilde{\x}_t^{\infty},\x_t^{\infty}) - B(\x_t^{\infty},\Tilde{\x}_{t-1}^{\infty})\\
    \iprod{\Tilde{\x}_t^{\infty}-\x}{\grad_t} &= B(\x,\Tilde{\x}_{t-1}^{\infty}) - B(\x,\Tilde{\x}_t^{\infty})-B(\Tilde{\x}_t^{\infty},\Tilde{\x}_{t-1}^{\infty}).
\end{align*}
Substituting this in the previous inequality gives us
\begin{align*}
f_t(\x_t) - f_t(\x) & \leq \iprod{\x_t-\x_t^{\infty}}{\grad_t} + \|\x_t^{\infty}-\Tilde{\x}_t^{\infty}\|\|\grad_t-g_{t}\|_{*} \\
&\quad + B(\Tilde{\x}_t^{\infty},\Tilde{\x}_{t-1}^{\infty}) - B(\Tilde{\x}_t^{\infty},\x_t^{\infty}) - B(\x_t^{\infty},\Tilde{\x}_{t-1}^{\infty})\vspace{0.1in}\\
&\quad +B(\x,\Tilde{\x}_{t-1}^{\infty}) - B(\x,\Tilde{\x}_t^{\infty})-B(\Tilde{\x}_t^{\infty},\Tilde{\x}_{t-1}^{\infty})\vspace{0.1in}\\
& = \iprod{\x_t-\x_t^{\infty}}{\grad_t} + \|\x_t^{\infty}-\Tilde{\x}_t^{\infty}\|\|\grad_t-g_{t}\|_{*} \\
&\quad + B(\x,\Tilde{\x}_{t-1}^{\infty}) - B(\x,\Tilde{\x}_t^{\infty}) - B(\Tilde{\x}_t^{\infty},\x_t^{\infty}) - B(\x_t^{\infty},\Tilde{\x}_{t-1}^{\infty})\vspace{0.1in}\\
& \stackrel{(a)}{\leq}\iprod{\x_t-\x_t^{\infty}}{\grad_t} + \|\x_t^{\infty}-\Tilde{\x}_t^{\infty}\|\|\grad_t-g_{t}\|_{*} \\
&\quad +  B(\x,\Tilde{\x}_{t-1}^{\infty}) - B(\x,\Tilde{\x}_t^{\infty}) - \frac{\eta\|\Tilde{\x}_t^{\infty}-\x_t^{\infty}\|^2}{2\stability }-\frac{ \eta\|\x_t^{\infty}-\Tilde{\x}_{t-1}^{\infty}\|^2}{2\stability },
\end{align*}
where $(a)$ follows from strongly convexity of $R$. 
Summing over $t=1,\dots T$, gives us
\begin{align*}
\sum_{t=1}^Tf_t(\x_t) - f_t(\x) &\leq \sum_{t=1}^T\iprod{\x_t-\x_t^{\infty}}{\grad_t} +  \underbrace{B(\x,\Tilde{\x}_{0}^{\infty}) - B(\x,\Tilde{\x}_{T}^{\infty})}_{S_1} \\
&\quad + \sum_{t=1}^T \|\x_t^{\infty}-\Tilde{\x}_t^{\infty}\|\|\grad_t-g_{t}\|_{*}\\
&\quad -\frac{\eta}{2\stability }\sum_{t=1}^T\left(\|\Tilde{\x}_t^{\infty}-\x_t^{\infty}\|^2 + \|\x_t^{\infty}-\Tilde{\x}_{t-1}^{\infty}\|^2\right).
\end{align*}
\paragraph{Bounding $S_1$.} We now bound $B(\x,\Tilde{\x}_{0}^{\infty}) - B(\x,\Tilde{\x}_{T}^{\infty})$. From the definition of $B$, we have
\begin{align*}
    B(\x,\Tilde{\x}_{0}^{\infty}) - B(\x,\Tilde{\x}_{T}^{\infty}) &= R(\Tilde{\x}_{T}^{\infty}) - \iprod{\grad_{1:T}}{\x-\Tilde{\x}_T^{\infty}}  -R(\Tilde{\x}_{0}^{\infty}) + \iprod{\grad_{1:0}}{\x-\Tilde{\x}_T^{\infty}}.
\end{align*}
Note that $\grad_{1:0} = 0.$ This gives us
\begin{align*}
    B(\x,\Tilde{\x}_{0}^{\infty}) - B(\x,\Tilde{\x}_{T}^{\infty}) &= R(\Tilde{\x}_{T}^{\infty}) - \iprod{\grad_{1:T}}{\x-\Tilde{\x}_T^{\infty}}  -R(\Tilde{\x}_{0}^{\infty}).
\end{align*}
We now use duality to convert the RHS of the above equation, which is currently in terms of $R$, into a quantity which depends on $\Phi$. From  Proposition~\ref{prop:ftpl_ftrl_connection} we have
\[
\Phi(g) = -\Psi(-g)=\inf_{\x\in\cX}\iprod{g}{\x} + R(\x).
\]
Since $\Tilde{\x}_{T}^{\infty}$ is the minimizer of $\iprod{\grad_{1:T}}{\x} + R(\x)$, we have $\Phi(\grad_{1:T}) =   \iprod{\grad_{1:T}}{\Tilde{\x}_T^{\infty}} + R(\Tilde{\x}_{T}^{\infty})$. Similarly, $\Phi(0) = R(\Tilde{\x}_0^{\infty}).$ Substituting these in the previous equation gives us
\begin{align*}
 B(\x,\Tilde{\x}_{0}^{\infty}) - B(\x,\Tilde{\x}_{T}^{\infty}) &= \Phi(\grad_{1:T}) - \iprod{\grad_{1:T}}{\x}  -\Phi(0)\\
 &=\Eover{\sigma}{\inf_{\x' \in \cX}\left\langle \grad_{1:T}-\sigma, \x'\right\rangle} - \iprod{\grad_{1:T}}{\x}  - \Eover{\sigma}{\inf_{\x' \in \cX}\left\langle -\sigma, \x'\right\rangle}\\
 &\leq \Eover{\sigma}{\left\langle \grad_{1:T}-\sigma, \x\right\rangle}- \iprod{\grad_{1:T}}{\x}  - \Eover{\sigma}{\inf_{\x' \in \cX}\left\langle -\sigma, \x'\right\rangle}\\
 &= \Eover{\sigma}{\inf_{\x' \in \cX}\left\langle \sigma, \x'\right\rangle} - \Eover{\sigma}{\left\langle \sigma, \x\right\rangle}\\
 &\leq D\Eover{\sigma}{\|\sigma\|_{*}} = \eta D
\end{align*}
\paragraph{Bounding Regret.} Substituting this in our regret bound and taking expectation on both sides gives us
\begin{align*}
    \E{\sum_{t=1}^Tf_t(\x_t) - f_t(\x)} &\leq \sum_{t=1}^T\E{\iprod{\x_t-\x_t^{\infty}}{\grad_t}} +  \eta D + \sum_{t=1}^T \E{\|\x_t^{\infty}-\Tilde{\x}_t^{\infty}\|\|\grad_t-g_{t}\|_{*}} \\ 
    &\quad -\frac{\eta}{2\stability }\sum_{t=1}^T\left(\E{\|\Tilde{\x}_t^{\infty}-\x_t^{\infty}\|^2} + \E{\|\x_t^{\infty}-\Tilde{\x}_{t-1}^{\infty}\|^2}\right)\\
     &\leq \sum_{t=1}^T\E{\iprod{\x_t-\x_t^{\infty}}{\grad_t}} +  \eta D + \sum_{t=1}^T \frac{\stability }{2\eta}\E{\|\grad_t-g_{t}\|_{*}^2} \\ 
    &\quad -\frac{\eta}{2\stability }\sum_{t=1}^T \E{\|\x_t^{\infty}-\Tilde{\x}_{t-1}^{\infty}\|^2}\\
\end{align*}
To finish the proof, we make use of the Holder's smoothness assumption on $f_t$ to bound the first term in the RHS above. From Holder's smoothness assumption, we have
\[
\iprod{\x_t-\x_t^{\infty}}{\grad_t - \grad f_t(\x_t^{\infty})} \leq L\|\x_t-\x_t^{\infty}\|^{1+\alpha}.
\]
Using this, we get
\begin{align*}
\E{\iprod{\x_t-\x_t^{\infty}}{\grad_t}|g_t, \x_{1:t-1},f_{1:t}} &\leq \E{\iprod{\x_t-\x_t^{\infty}}{\grad f_t(\x_t^{\infty})} + L\|\x_t-\x_t^{\infty}\|^{1+\alpha}|g_t,\x_{1:t-1},f_{1:t}}\\
& \stackrel{(a)}{=} L\E{\|\x_t-\x_t^{\infty}\|^{1+\alpha}|g_t,\x_{1:t-1},f_{1:t}}\\
& \stackrel{(b)}{\leq} \normcompf ^{1+\alpha}L\E{\|\x_t-\x_t^{\infty}\|_2^{1+\alpha}|g_t,\x_{1:t-1},f_{1:t}}\\
& \stackrel{(c)}{\leq} \normcompf ^{1+\alpha}L\E{\|\x_t-\x_t^{\infty}\|_2^2|g_t,\x_{1:t-1},f_{1:t}}^{(1+\alpha)/2}\\
& \stackrel{(d)}{\leq} L\left(\frac{\normcompf \normcompb D}{\sqrt{m}}\right)^{1+\alpha},
\end{align*}
where $(a)$ follows from the fact that $\E{\iprod{\x_t-\x_t^{\infty}}{\grad f_t(\x_t^{\infty})}|g_t,\x_{1:t-1},f_{1:t}} = 0$, $(b)$ follows from the definition of norm compatibility constant $\normcompf $, $(c)$ follows from Holders inequality and $(d)$ uses the fact that conditioned on $\{g_t, \x_{1:t-1},f_{1:t}\}$, $\x_t-\x_t^{\infty}$ is the average of $m$  i.i.d bounded mean $0$ random variables, the variance of which scales as $O(D^2/m)$. Substituting this in the above regret bound gives us the required result.
\subsection{Proof of Corollary~\ref{cor:ftpl_cvx_gaussian}}
We first bound $\Eover{\sigma}{\|\sigma\|_2}$. Relying on spherical symmetry of the perturbation distribution and the fact that the density of $\pertdist$ on the spherical shell of radius $r$ is proportional to $r^{d-1}$, we get
\begin{align*}
    \Eover{\sigma}{\|\sigma\|_2} = \frac{\int_{r=0}^{(1+d^{-1})\eta} r\times r^{d-1} dr}{\int_{r=0}^{(1+d^{-1})\eta} r^{d-1} dr} = \eta.
\end{align*}
We now bound the stability of predictions of OFTPL. Our technique for bounding the stability uses similar arguments as ~\citet{hazan2020projection} (see Lemma 4.2 of \citep{hazan2020projection}).  Recall, to bound stability, we need to show that $\Phi(g) = \Eover{\sigma}{\inf_{\x \in \cX}\left\langle g-\sigma, \x\right\rangle}$ is smooth. 
Let $\phi_0(g) = \inf_{\x\in\cX}\iprod{g}{\x-\x_{00}}$, where $\x_{00}$ is an arbitrary point in $\cX$. We can rewrite $\Phi(g)$ as  $$\Phi(g) = \Eover{\sigma}{\phi_0(g-\sigma)} + \iprod{g}{\x_{00}}.$$ Since the second term in the RHS above is linear in $g$, any upper bound on the smoothness of $\Eover{\sigma}{\phi_0(g-\sigma)}$ is also a bound on the smoothness of $\Phi(g)$. So we focus on bounding the smoothness of $\Eover{\sigma}{\phi_0(g-\sigma)}$.

First note that $\phi_0(g)$ is $D$ Lipschitz and satisfies the following for any $g_1,g_2\in\mathbb{R}^d$
\begin{align*}
\phi_0(g_1) - \phi_0(g_2) & = \inf_{\x\in\cX}\iprod{-g_2}{\x-\x_{00}}-\inf_{\x\in\cX}\iprod{-g_1}{\x-\x_{00}}\\
& \leq \sup_{\x\in\cX}\iprod{g_1-g_2}{\x-\x_{00}}\\
&\leq D\|g_1-g_2\|_2.
\end{align*}
Letting $\Phi_0(g) = \Eover{\sigma}{\phi_0(g-\sigma)}$, 
Lemma 4.2 of \citet{hazan2020projection} shows that $\Phi_0(g)$ is smooth and satisfies
\[
\|\grad\Phi_0(g_1) - \grad\Phi_0(g_2)\|_2 \leq dD\eta^{-1}\|g_1-g_2\|_2.
\]
This shows that the predictions of OFTPL are $dD\eta^{-1}$ stable. 
The rest of the proof involves substituting $C=dD$ in the regret bound of Theorem~\ref{thm:oftpl_regret} and setting $g_t = 0$ and using the fact that $\|\grad_t\|_2\leq G$.
\section{Online Nonconvex Learning}
\subsection{Proof of Theorem~\ref{thm:oftpl_noncvx_regret}}
Before we present the proof of the Theorem, we introduce some notation and present some useful intermediate results. We note that unlike the convex case, there are no know Fenchel duality theorems for infinite dimensional setting. So more careful arguments are need to obtain tight regret bounds. Our proof mimics the proof of Theorem~\ref{thm:oftpl_regret}.
\subsubsection{Notation}
Let $\cP$ be the set of all probability measures on $\cX$. We define functions $\Phi:\cF\to \mathbb{R}$, $R:\cP\to\mathbb{R}$ as follows
\begin{align*}
    &\Phi(f) = \Eover{\sigma}{\inf_{P \in \cP}\Eover{\x\sim P}{f(\x)-\sigma(\x)}},\\
    &R(P) = \sup_{f\in \cF} -\Eover{\x \sim P}{f(\x)} + \Phi(f).
\end{align*}
 Also, note that the function $\gphiNA:\cF\to\cP$  defined in Section~\ref{sec:online_noncvx}  can be written as
\begin{align*}
    \gphi{f} = \Eover{\sigma}{\argmin_{P \in \cP}\Eover{\x\sim P}{f(\x)-\sigma(\x)}}.
\end{align*}
Note that, $\gphi{f}$ is well defined because from our assumption on the perturbation distribution, the minimization problem inside the expectation has a unique minimizer with probability one.
To simplify the notation, in the sequel, we use the shorthand notation $\iprod{P}{f}$ to denote $\Eover{\x\sim P}{f(\x)}$, for any $P\in\cP$ and $f\in \cF$. Similarly, for any $P_1,P_2\in \cP$ and $f\in \cF$, we use the notation $\iprod{P_1-P_2}{f}$ to denote $\Eover{\x\sim P_1}{f(\x)}-\Eover{\x\sim P_2}{f(\x)}$.
\subsubsection{Intermediate Results}
\begin{lemma}
\label{lem:noncvx_gradient}
For any $g\in \cF$, $R(\gphi{g}) = -\iprod{\gphi{g}}{g} + \Phi(g)$.
\end{lemma}
\begin{proof}
Define $P_{g,\sigma}$ as $$P_{g,\sigma}=\argmin_{P \in \cP}\Eover{\x\sim P}{g(\x)-\sigma(\x)}.$$ Note that $\gphi{g} = \Eover{\sigma}{P_{g,\sigma}}$.
For any $g,h \in \cF$, we have
\begin{align*}
    \Phi(h) &=  \Eover{\sigma}{\inf_{P \in \cP}\iprod{P}{h-\sigma}}\\
    &\leq \Eover{\sigma}{\iprod{P_{g,\sigma}}{h-\sigma}}\\
    &= \Eover{\sigma}{\iprod{P_{g,\sigma}}{g-\sigma}} + \Eover{\sigma}{\iprod{P_{g,\sigma}}{h-g}}\\
    & = \Phi(g) + \iprod{\gphi{g}}{h-g}.
\end{align*}
This shows that for any $g,h \in \cF$
\begin{equation}
    \label{eqn:noncvx_phi_convex}
    \Phi(h) - \iprod{\gphi{g}}{h} \leq \Phi(g) - \iprod{\gphi{g}}{g}.
\end{equation}
Taking supremum over $h$ of the LHS quantity gives us
\[
R(\gphi{g})=\sup_{h\in \cF}\Phi(h) - \iprod{\gphi{g}}{h} = \Phi(g) - \iprod{\gphi{g}}{g}.
\]
\end{proof}
\begin{lemma}[Strong Smoothness]
\label{lem:noncvx_phi_smooth}
The function $-\Phi$ is convex and strongly smooth and satisfies the following inequality for any $g_1,g_2\in \cF$
\[
-\Phi(g_2)\leq -\Phi(g_1) - \iprod{\gphi{g_1}}{g_2-g_1} + \frac{C}{2\eta}\|g_2-g_1\|_{\cF}^2.
\]
\end{lemma}
\begin{proof}
Let $g_1,g_2\in\cF$ and $\alpha \in [0,1]$. Then 
\begin{align*}
    \Phi(\alpha g_1 + (1-\alpha)g_2) &= \Eover{\sigma}{\inf_{P \in \cP}\iprod{P}{\alpha g_1 + (1-\alpha)g_2-\sigma}}\\
    & \geq \alpha\Eover{\sigma}{\inf_{P \in \cP}\iprod{P}{g_1-\sigma}} + (1-\alpha)\Eover{\sigma}{\inf_{P \in \cP}\iprod{P}{g_2-\sigma}}\\
    & = \alpha\Phi(g_1) + (1-\alpha)\Phi(g_2).
\end{align*}
This shows that $-\Phi$ is convex. To show smoothness, we rely on the following stability property
\[
\forall g_1,g_2\in\cF \quad \gammaF(\gphi{g_1}, \gphi{g_2}) \leq \frac{C}{\eta} \|g_1-g_2\|_{\cF}.
\]
Let $T$ be an arbitrary positive integer and for $t\in \{0,1,\dots T\}$, define $\alpha_t = t/T$. Let $h = g_2-g_1$. We have
\begin{align*}
    \Phi(g_1)-\Phi(g_2) &= \Phi(g_1 + \alpha_0h)-\Phi(g_1 + \alpha_Th)\\
    &=\sum_{t=0}^{T-1}\left(\Phi(g_1 + \alpha_{t}h)-\Phi(g_1 + \alpha_{t+1}h)\right)
\end{align*}
Since $-\Phi$ is convex and satisfies Equation~\eqref{eqn:noncvx_phi_convex}, we have
\begin{align*}
    \Phi(g_1)-\Phi(g_2) &=\sum_{t=0}^{T-1}\left(\Phi(g_1 + \alpha_{t}h)-\Phi(g_1 + \alpha_{t+1}h)\right)\\
    &\leq -\sum_{t=0}^{T-1} \frac{1}{T}\iprod{\gphi{g_1+\alpha_{t+1}h}}{h}
\end{align*}
Using stability, we get
\begin{align*}
    \Phi(g_1)-\Phi(g_2) &\leq -\sum_{t=0}^{T-1} \frac{1}{T}\iprod{\gphi{g_1+\alpha_{t+1}h}}{h}\\
    & = \sum_{t=0}^{T-1} \frac{1}{T}\left(\iprod{\gphi{g_1}-\gphi{g_1+\alpha_{t+1}h}}{h} - \iprod{\gphi{g_1}}{h}\right)\\
    & \stackrel{(a)}{\leq} -\iprod{\gphi{g_1}}{h} + \sum_{t=0}^{T-1} \frac{1}{T}\gammaF(\gphi{g_1},\gphi{g_1+\alpha_{t+1}h})\|h\|_{\cF} \\
    &\stackrel{(b)}{\leq} -\iprod{\gphi{g_1}}{h} + \sum_{t=0}^{T-1} \frac{C}{T\eta}\|\alpha_{t+1}h\|_{\cF}\|h\|_{\cF}\\
    &=-\iprod{\gphi{g_1}}{h} + \sum_{t=0}^{T-1} \frac{C\alpha_{t+1}}{T\eta}\|h\|^2_{\cF}\\
    &=-\iprod{\gphi{g_1}}{h} +\frac{C}{\eta}\frac{T+1}{2T}\|h\|^2_{\cF},
\end{align*}
where $(a)$ follows from the definition of $\gammaF$ and $(b)$ follows from the stability assumption. Taking $T\to\infty$, we get
\[
-\Phi(g_2)\leq -\Phi(g_1) - \iprod{\gphi{g_1}}{g_2-g_1} + \frac{C}{2\eta}\|g_2-g_1\|_{\cF}^2.
\]
\end{proof}

\begin{lemma}[Strong Convexity]
\label{lem:noncvx_reg_strong_cvx}
For any $P\in \cP$ and $g \in \cF$, $R$ satisfies the following inequality
\[
R(P)\geq R(\gphi{g}) +\iprod{\gphi{g}-P}{g}  + \frac{\eta}{2\stability }\gammaF(P,\gphi{g})^2.
\]
\end{lemma}
\begin{proof}
From Lemma~\ref{lem:noncvx_phi_smooth} we know that the following holds for any $g,h\in\cF$
\[
\Phi(g)\geq \underbrace{\Phi(h) + \iprod{\gphi{h}}{g-h} - \frac{C}{2\eta}\|g-h\|_{\cF}^2}_{\Phi_{\text{lb}, h}(g)}.
\]
Define $R_{\text{lb},h}(P)$ as
\[
R_{\text{lb}}(P) =\sup_{g\in\cF} -\iprod{P}{g} + \Phi_{\text{lb},h}(g).
\]
Since $\Phi(g) \geq \Phi_{\text{lb},h}(g)$ for all $g\in\cF$,  $R(P) \geq R_{\text{lb},h}(P)$ for all $P$. We now derive an expression for $R_{\text{lb},h}(P)$.
Note that from Lemma~\ref{lem:noncvx_gradient} we have $R(\gphi{h}) = -\iprod{\gphi{h}}{h} + \Phi(h)$. Using this, we get
\begin{align*}
    R_{\text{lb},h}(P) &= \sup_{g\in\cF} -\iprod{P}{g} + \Phi_{\text{lb},h}(g)\\
    & \stackrel{(a)}{=} \sup_{g\in\cF} \left(-\iprod{P}{g} + \Phi(h) + \iprod{\gphi{h}}{g-h} - \frac{C}{2\eta}\|g-h\|_{\cF}^2\right)\\
    &\stackrel{(b)}{=}R(\gphi{h}) + \sup_{g\in\cF} \left(\iprod{\gphi{h}-P}{g}-\frac{C}{2\eta}\|g-h\|_{\cF}^2\right),
\end{align*}
where $(a)$ follows from the definition of $\Phi_{\text{lb},h}(g)$ and $(b)$ follows from Lemma~\ref{lem:noncvx_gradient}. We now do a change of variables in the supremum of the above expression. Substituting $g' = g - h$, we get
\begin{align*}
    R_{\text{lb},h}(P) & = R(\gphi{h}) + \iprod{\gphi{h}-P}{h} + \sup_{g'\in\cF} \left(\iprod{\gphi{h}-P}{g'}-\frac{C}{2\eta}\|g'\|_{\cF}^2\right).
\end{align*}
We now show that 
\[
\sup_{g'\in\cF} \left(\iprod{\gphi{h}-P}{g'}-\frac{C}{2\eta}\|g'\|_{\cF}^2\right) \geq \frac{\eta}{2C}\gammaF(P,\gphi{h})^2.
\]
To this end, we choose a $g'' \in \cF$ such that 
\begin{equation}
\label{eqn:noncvx_sc_g}
   \|g''\|_{\cF} = \frac{\eta}{C}\gammaF(P,\gphi{h}),\quad  \iprod{\gphi{h}-P}{g''} = \frac{\eta}{C}\gammaF(P,\gphi{h})^2. 
\end{equation}
If such a $g''$ can be found, we have
\begin{align*}
    \sup_{g'\in\cF} \left(\iprod{\gphi{h}-P}{g'}-\frac{C}{2\eta}\|g'\|_{\cF}^2\right) &\geq \iprod{\gphi{h}-P}{g''}-\frac{C}{2\eta}\|g''\|_{\cF}^2\\
    & = \frac{\eta}{2C}\gammaF(P,\gphi{h})^2.
\end{align*}
This would then imply the main claim of the Lemma.
\begin{align*}
    R(P) \geq R_{\text{lb},h}(P) \geq R(\gphi{h}) + \iprod{\gphi{h}-P}{h} + \frac{\eta}{2C}\gammaF(P,\gphi{h})^2.
\end{align*}
\paragraph{Finding $g''$.} We now construct a $g''$ which satisfies Equation~\eqref{eqn:noncvx_sc_g}. From the definition of $\gammaF$ we know that 
\[
\gammaF(P,\gphi{h}) = \sup_{\|g'\|_{\cF}\leq 1} |\iprod{\gphi{h}-P}{g'}|
\]
Suppose the supremum is achieved at $g^*$.
Define $g''$ as $\frac{\eta s}{C}\gammaF(P,\gphi{h})g^*$, where $s = \text{sign}(\iprod{\gphi{h}-P}{g^*})$. It can be easily verified that $g''$ satifies Equation~\eqref{eqn:noncvx_sc_g}.

If the supremum is never achieved, the same argument as above can still be made using a sequence of functions $\{g_{n}\}_{n=1}^{\infty}$ such that $$\|g_n\|_{\cF}\leq 1,\quad  \lim_{n\to\infty} |\iprod{\gphi{h}-P}{g_n}| = \gammaF(P,\gphi{h}).$$ Define $g''_n$ as $\frac{\eta s_n}{C}\gammaF(P,\gphi{h})g_n$, where $s_n = \text{sign}(\iprod{\gphi{h}-P}{g_n})$. Since $\lim_{n\to\infty}\|g_n\|_{\cF} = 1$, we have $\lim_{n \to \infty}\|g''_{n}\|_{\cF} = \frac{\eta}{C} \gammaF(P,\gphi{h})$. Moreover, $$\lim_{n\to\infty} \iprod{\gphi{h}-P}{g''_n} = \lim_{n\to\infty} \frac{\eta}{C}\gammaF(P,\gphi{h}) \Big|\iprod{\gphi{h}-P}{g_n}\Big| = \frac{\eta}{C}\gammaF(P,\gphi{h})^2.$$
This shows that
\begin{align*}
    \sup_{g'\in\cF} \left(\iprod{\gphi{h}-P}{g'}-\frac{C}{2\eta}\|g'\|_{\cF}^2\right) &\geq \lim_{n\to\infty}\iprod{\gphi{h}-P}{g''_n}-\frac{C}{2\eta}\|g''_n\|_{\cF}^2\\
    & = \frac{\eta}{2C}\gammaF(P,\gphi{h})^2.
\end{align*}
This finishes the proof of the Lemma.
\end{proof}
\subsubsection{Main Argument}
We are now ready to prove Theorem~\ref{thm:oftpl_noncvx_regret}. Our proof relies on Lemma~\ref{lem:noncvx_reg_strong_cvx} and uses similar arguments as used in the proof of Theorem~\ref{thm:oftpl_regret}.
We first rewrite $P_t, \Tilde{P}_t$ as
\begin{align*}
  P_t &= \frac{1}{m}\sum_{j=1}^m\argmin_{P \in \cP}\Eover{\x\sim P}{\sum_{i = 1}^{t-1}f_i(\x)+g_t(\x)-\sigma_{t,j}(\x)},\\
  \Tilde{P}_t &= \frac{1}{m}\sum_{j=1}^m\argmin_{P \in \cP}\Eover{\x\sim P}{\sum_{i = 1}^{t}f_i(\x)-\sigma'_{t,j}(\x)}.
\end{align*}
Note that
\begin{align*}
  P_t^{\infty} &= \E{P_t|g_t,f_{1:t-1},P_{1:t-1}} = \gphi{f_{1:t-1} + g_t},\\
  \Tilde{P}_t^{\infty} &= \E{\Tilde{P}_t|f_{1:t-1},P_{1:t-1}} = \gphi{f_{1:t}},
\end{align*}
with $P_1^{\infty} = \Tilde{P}_0^{\infty} =\gphi{0}$.
Define functions $B(\cdot,P_t^{\infty}), B(\cdot, \Tilde{P}_t^{\infty})$ as
\begin{align*}
B(P,P_t^{\infty}) &= R(P) - R(P_t^{\infty}) + \iprod{P-P_t^{\infty}}{f_{1:t-1}+g_t},\\
B(P,\Tilde{P}_t^{\infty}) &= R(P) - R(\Tilde{P}_t^{\infty}) + \iprod{P-\Tilde{P}_t^{\infty}}{f_{1:t}}.
\end{align*}
From Lemma~\ref{lem:noncvx_reg_strong_cvx}, we have
\[
B(P,P_t^{\infty}) \geq \frac{\eta}{2C}\gammaF(P, P_t^{\infty})^2,\quad B(P,\Tilde{P}_t^{\infty}) \geq \frac{\eta}{2C}\gammaF(P, \Tilde{P}_t^{\infty})^2.
\]
For any $P\in \cP$, we have
\begin{align*}
\E{f_t(\x_t)-f_t(P)} &= \E{f_t(P_t)-f_t(P)} \\
& = \E{\iprod{P_t-P}{f_t}}\\
&=\E{\iprod{P_t-P_t^{\infty}}{f_t}} + \E{\iprod{P_t^{\infty}-P}{f_t}}\\
&=\E{\iprod{P_t-P_t^{\infty}}{f_t}} + \E{\iprod{P_t^{\infty}-\Tilde{P}_t^{\infty}}{f_t-g_t}} \\
&\quad+ \E{\iprod{P_t^{\infty}-\Tilde{P}_t^{\infty}}{g_t}}+\E{\iprod{\Tilde{P}_t^{\infty}-P}{f_t}}\\
&\stackrel{(a)}{\leq}\E{\gammaF(P_t^{\infty}, \Tilde{P}_t^{\infty})\|f_t-g_t\|_{\cF}}+ \E{\iprod{P_t^{\infty}-\Tilde{P}_t^{\infty}}{g_t}} \\
&\quad+\E{\iprod{\Tilde{P}_t^{\infty}-P}{f_t}},
\end{align*}
where $(a)$ follows from the fact that $\E{\iprod{P_t-P_t^{\infty}}{f_t}|g_t, f_{1:t-1}, P_{1:t-1}} = 0$ and as a result $\E{\iprod{P_t-P_t^{\infty}}{f_t}}=0$. Next, a simple calculation shows that
\begin{align*}
    \iprod{P_t^{\infty}-\Tilde{P}_t^{\infty}}{g_{t}} &= B(\Tilde{P}_t^{\infty},\Tilde{P}_{t-1}^{\infty}) - B(\Tilde{P}_t^{\infty},P_t^{\infty}) - B(P_t^{\infty},\Tilde{P}_{t-1}^{\infty})\\
    \iprod{\Tilde{P}_t^{\infty}-P}{f_t} &= B(P,\Tilde{P}_{t-1}^{\infty}) - B(P,\Tilde{P}_t^{\infty})-B(\Tilde{P}_t^{\infty},\Tilde{P}_{t-1}^{\infty}).
\end{align*}
Substituting this in the previous regret bound gives us
\begin{align*}
\E{f_t(\x_t)-f_t(P)} & \leq  \E{\gammaF(P_t^{\infty}, \Tilde{P}_t^{\infty})\|f_t-g_t\|_{\cF}}  + \E{B(\Tilde{P}_t^{\infty},\Tilde{P}_{t-1}^{\infty}) - B(\Tilde{P}_t^{\infty},P_t^{\infty}) - B(P_t^{\infty},\Tilde{P}_{t-1}^{\infty})}\\
&\quad +\E{B(P,\Tilde{P}_{t-1}^{\infty}) - B(P,\Tilde{P}_t^{\infty})-B(\Tilde{P}_t^{\infty},\Tilde{P}_{t-1}^{\infty})}\\
& = \E{\gammaF(P_t^{\infty}, \Tilde{P}_t^{\infty})\|f_t-g_t\|_{\cF}} \\
&\quad + \E{B(P,\Tilde{P}_{t-1}^{\infty}) - B(P,\Tilde{P}_t^{\infty}) - B(\Tilde{P}_t^{\infty},P_t^{\infty}) - B(P_t^{\infty},\Tilde{P}_{t-1}^{\infty})}\\
& \stackrel{(a)}{\leq}\E{\gammaF(P_t^{\infty}, \Tilde{P}_t^{\infty})\|f_t-g_t\|_{\cF}} \\
&\quad +  \E{B(P,\Tilde{P}_{t-1}^{\infty}) - B(P,\Tilde{P}_t^{\infty})} - \E{\frac{\eta}{2C}\gammaF(\Tilde{P}_t^{\infty}, P_t^{\infty})^2 + \frac{\eta}{2C}\gammaF(P_t^{\infty},\Tilde{P}_{t-1}^{\infty})^2 }\\
& \stackrel{(b)}{\leq} \frac{C}{2\eta}\E{\|f_t-g_t\|_{\cF}^2} +  \E{B(P,\Tilde{P}_{t-1}^{\infty}) - B(P,\Tilde{P}_t^{\infty})} - \E{ \frac{\eta}{2C}\gammaF(P_t^{\infty},\Tilde{P}_{t-1}^{\infty})^2 }
\end{align*}
where $(a)$ follows from Lemma~\ref{lem:noncvx_reg_strong_cvx}, and $(b)$ uses the fact that $|xy|\leq \frac{1}{2c}|x|^2 + \frac{c}{2}|y|^2$, for any $x,y$, $c> 0$. Summing over $t=1,\dots T$ gives us
\begin{align*}
    \sum_{t=1}^T \E{f_t(\x_t)-f_t(P)} &\leq  \underbrace{\E{B(P,\Tilde{P}_{0}^{\infty}) - B(P,\Tilde{P}_T^{\infty})}}_{S_1}+ \sum_{t=1}^T\frac{C}{2\eta}\E{\|f_t-g_t\|_{\cF}^2}\\
    &\quad- \sum_{t=1}^T\frac{\eta}{2C}\E{ \gammaF(P_t^{\infty},\Tilde{P}_{t-1}^{\infty})^2 }
\end{align*}
To finish the proof of the Theorem, we need to bound $S_1$.
\paragraph{Bounding $S_1$.}  From the definition of $B$, we have
\begin{align*}
    B(P,\Tilde{P}_{0}^{\infty}) - B(P,\Tilde{P}_T^{\infty}) &= R(\Tilde{P}_{T}^{\infty}) - \iprod{P-\Tilde{P}_T^{\infty}}{f_{1:T}}-R(\Tilde{\x}_{0}^{\infty}),
\end{align*}
where we used the fact that $f_{1:0} = 0$.
We now rely on Lemma~\ref{lem:noncvx_gradient} to convert the above equation, which is currently in terms of $R$, into a quantity which depends on $\Phi$. Using Lemma~\ref{lem:noncvx_gradient}, we get
\begin{align*}
    B(P,\Tilde{P}_{0}^{\infty}) - B(P,\Tilde{P}_T^{\infty}) &=  \Phi(f_{1:T}) - \iprod{P}{f_{1:T}}-\Phi(0).
\end{align*}
From the definition of $\Phi$ we have
\begin{align*}
 B(P,\Tilde{P}_{0}^{\infty}) - B(P,\Tilde{P}_T^{\infty}) &= \Phi(f_{1:T}) - \iprod{P}{f_{1:T}}-\Phi(0)\\
 &=\Eover{\sigma}{\inf_{P' \in \cP}\iprod{P'}{f_{1:T}-\sigma}} - \iprod{P}{f_{1:T}}  - \Eover{\sigma}{\inf_{P' \in \cP}\iprod{P'}{-\sigma}}\\
 &\leq \Eover{\sigma}{\iprod{P}{f_{1:T}-\sigma}} - \iprod{P}{f_{1:T}}  - \Eover{\sigma}{\inf_{P' \in \cP}\iprod{P'}{-\sigma}}\\
 &= \Eover{\sigma}{\sup_{P' \in \cP}\iprod{P'}{\sigma}} - \Eover{\sigma}{\left\langle P,\sigma\right\rangle}\\
 &\leq D\Eover{\sigma}{\|\sigma\|_{\cF}} = \eta D,
\end{align*}
where the last inequality follows from our bound on the diameter of $\cP$. Substituting this in the above regret bound gives us the required result.
\subsection{Proof of Corollary~\ref{cor:ftpl_noncvx_exp}}
To prove the corollary we first show that for our choice of perturbation distribution, $\argmin_{\x\in\cX} f(\x) - \sigma(\x)$ has a unique minimizer with probability one, for any $f\in\cF$. Next, we show that the predictions of OFTPL are stable.
\subsubsection{Intermediate Results}
\begin{lemma}[Unique Minimizer]
Suppose the perturbation function is such that $\sigma(\x) = \iprod{\bar{\sigma}}{\x}$, where $\bar{\sigma} \in \mathbb{R}^d$ is a random vector whose entries are sampled independently from $\text{Exp}(\eta)$. Then, for any $f\in\cF$, $\argmin_{\x\in\cX} f(\x) - \sigma(\x)$ has a unique minimizer with probability one.
\end{lemma}
\begin{proof}
Define $\x_f(\sigma)$ as
\[
\x_f(\bar{\sigma}) \in \argmin_{\x\in\cX} f(\x) - \iprod{\bar{\sigma}}{\x}.
\]
For any $\bar{\sigma}_1,\bar{\sigma}_2$ we now show that $\x_f(\bar{\sigma})$ satisfies the following monotonicity property $$\iprod{\x_f(\bar{\sigma}_1)-\x_f(\bar{\sigma}_2)}{\bar{\sigma}_1-\bar{\sigma}_2} \geq 0.$$
From the optimality of $\x_f(\bar{\sigma}_1),\x_f(\bar{\sigma}_2)$ we have
\begin{align*}
    f(\x_f(\bar{\sigma}_1)) - \iprod{\bar{\sigma}_1}{\x_f(\bar{\sigma}_1)} &\leq f(\x_f(\bar{\sigma}_2)) - \iprod{\bar{\sigma}_1}{\x_f(\bar{\sigma}_2)}\\
    & = f(\x_f(\bar{\sigma}_2)) - \iprod{\bar{\sigma}_2}{\x_f(\bar{\sigma}_2)} + \iprod{\bar{\sigma}_2-\bar{\sigma}_1}{\x_f(\bar{\sigma}_2)}\\
    &\leq f(\x_f(\bar{\sigma}_1)) - \iprod{\bar{\sigma}_2}{\x_f(\bar{\sigma}_1)}+ \iprod{\bar{\sigma}_2-\bar{\sigma}_1}{\x_f(\bar{\sigma}_2)}.
\end{align*}
This shows that $\iprod{\bar{\sigma}_2-\bar{\sigma}_1}{\x_f(\bar{\sigma}_2)-\x_f(\bar{\sigma}_1)} \geq 0$. To finish the proof of Lemma, we rely on Theorem 1 of~\citet{zarantonello1973dense}, which shows that the set of points for which a monotone operator is not single-valued has Lebesgue measure zero. Since  the distribution of $\bar{\sigma}$ is absolutely continuous w.r.t Lebesgue measure, this shows that $\argmin_{\x\in\cX} f(\x) - \sigma(\x)$ has a unique minimizer with probability one.
\end{proof}
\subsubsection{Main Argument}
For our choice of perturbation distribution, $\Eover{\sigma}{\|\sigma\|_{\cF}} = \Eover{\bar{\sigma}}{\|\bar{\sigma}\|_{\infty}} = \eta\log{d}$. We now bound the stability of predictions of OFTPL. First note that for our choice of primal space $(\cF,\|\cdot\|_{\cF})$, $\gammaF$ is the Wasserstein-1 metric, which is defined as
\[
\gammaF(P_1,P_2) = \sup_{f\in\cF, \|f\|_{\cF}\leq 1} \Big|\Eover{\x\sim P_1}{f(\x)}-\Eover{\x\sim P_2}{f(\x)}\Big| = \inf_{Q\in\Gamma(P_1,P_2)}\Eover{(\x_1,\x_2)\sim Q}{\|\x_1-\x_2\|_1},
\]
where $\Gamma(P_1,P_2)$ is the set of all probability measures on $\cX\times\cX$ with marginals $P_1,P_2$   on the first and second factors respectively.
 Define $\x_f(\bar{\sigma})$ as
\[
\x_f(\bar{\sigma}) \in \argmin_{\x\in\cX} f(\x) - \iprod{\bar{\sigma}}{\x}.
\]
Note that $\gphi{f}$ is the distribution of random variable $\x_f(\bar{\sigma})$. \citet{suggala2019online} show that for any $f,g\in\cF$
\[
\Eover{\bar{\sigma}}{\|\x_{f}(\bar{\sigma})-\x_{g}(\bar{\sigma})\|_1} \leq \frac{125d^2D}{\eta}\|f-g\|_{\cF}.
\]
Since $\gammaF(\gphi{f},\gphi{g}) \leq \Eover{\bar{\sigma}}{\|\x_{f}(\bar{\sigma})-\x_{g}(\bar{\sigma})\|_1}$, this shows that OFTPL is $\order{d^2D\eta^{-1}}$ stable w.r.t $\|\cdot\|_{\cF}$. Substituting the stability bound in the regret bound of Theorem~\ref{thm:oftpl_noncvx_regret} shows that
\begin{align*}
    \sup_{P\in\cP}\E{\sum_{t=1}^Tf_t(\x_t)-f_t(P)} &= \eta D\log{d}  \\
&\quad +\order{ \sum_{t=1}^T \frac{d^2D }{\eta}\E{\|f_t-g_{t}\|_{\cF}^2}  -\sum_{t=1}^T \frac{\eta}{d^2D }\E{\gammaF(P_t^{\infty},\Tilde{P}_{t-1}^{\infty})^2}}.
\end{align*}
\section{Convex-Concave Games}\label{sec:cvx-games}
Our algorithm for convex-concave games is presented in Algorithm~\ref{alg:oftpl_cvx_games}. Before presenting the proof of Theorem~\ref{thm:oftpl_cvx_smooth_games_uniform}, we first present a more general result in Section~\ref{sec:cvx_games_general}.  Theorem~\ref{thm:oftpl_cvx_smooth_games_uniform} immediately follows from our general result by instantiating it for the uniform noise distribution.
\begin{algorithm}[t]
\caption{OFTPL for convex-concave games}
\label{alg:oftpl_cvx_games}
\begin{algorithmic}[1]
  \small
  \State \textbf{Input:}  Perturbation Distributions $\pertdist^1,\pertdist^2$ of $\x,\y$ players, number of samples $m,$ iterations $T$
  \For{$t = 1 \dots T$}
  \If{$t=1$}
  \State Sample $\{\sigma_{1,j}^1\}_{j=1}^m,$ $\{\sigma_{1,j}^2\}_{j=1}^m$ from $\pertdist^1,\pertdist^2$
  \State  $\x_1 = \frac{1}{m}\sum_{j=1}^m\left[\argmin_{\x\in\cX} \iprod{-\sigma_{1,j}^1}{\x}\right], \y_1 = \frac{1}{m}\left[\sum_{j=1}^m\argmax_{\y\in\cY} \iprod{\sigma_{1,j}^2}{\y}\right]$
  \State \textbf{continue}
  \EndIf
  \State \texttt{//Compute guesses}
  \For{$j=1\dots m$}
  \State Sample $\sigma_{t,j}^1\sim \pertdist^1, \sigma_{t,j}^2\sim \pertdist^2$
  \State $\Tilde{\x}_{t-1,j}= \underset{\x \in \cX}{\argmin}\iprod{\sum_{i = 1}^{t-1}\grad_{\x}f(\x_i,\y_i) -\sigma_{t,j}^1}{\x}$
  \State $\Tilde{\y}_{t-1,j}= \underset{\y \in \cY}{\argmax}\iprod{\sum_{i = 1}^{t-1}\grad_{\y}f(\x_i,\y_i) +\sigma_{t,j}^2}{\y}$
  \EndFor
  \State $\Tilde{\x}_{t-1} = \frac{1}{m}\sum_{j=1}^m\Tilde{\x}_{t-1,j}$, $\Tilde{\y}_{t-1} = \frac{1}{m}\sum_{j=1}^m\Tilde{\y}_{t-1,j}$
  \State \texttt{//Use the guesses to compute the next action}
  \For{$j=1\dots m$}
  \State Sample $\sigma_{t,j}^1\sim \pertdist^1, \sigma_{t,j}^2\sim \pertdist^2$
  \State $\x_{t,j}= \underset{\x \in \cX}{\argmin}\iprod{\sum_{i = 1}^{t-1}\grad_{\x}f(\x_i,\y_i)+  \grad_{\x}f(\Tilde{\x}_{t-1},\Tilde{\y}_{t-1})-\sigma_{t,j}^1}{\x}$
  \State $\y_{t,j}= \underset{\y \in \cY}{\argmax}\iprod{\sum_{i = 1}^{t-1}\grad_{\y}f(\x_i,\y_i)+  \grad_{\y}f(\Tilde{\x}_{t-1},\Tilde{\y}_{t-1})+\sigma_{t,j}^2}{\y}$
  \EndFor
  \State  $\x_t=\frac{1}{m}\sum_{j=1}^m \x_{t,j}, \y_t=\frac{1}{m}\sum_{j=1}^m \y_{t,j}$
  \EndFor
  \State \Return $\{(\x_t, \y_t)\}_{t=1}^T$
\end{algorithmic}
\end{algorithm}

\subsection{General Result}
\label{sec:cvx_games_general}
\begin{theorem}
\label{thm:oftpl_cvx_smooth_games}
Consider the minimax game in Equation~\eqref{eqn:minimax_game}. Suppose $f$ is convex in $\x$, concave in $\y$ and is Holder smooth w.r.t some norm $\|\cdot\|$
\begin{align*}
    \|\grad_\x f(\x,\y)-\grad_\x f(\x',\y')\|_{*} \leq L_1\|\x-\x'\|^{\alpha} + L_2\|\y-\y'\|^{\alpha},\\
    \|\grad_\y f(\x,\y)-\grad_\y f(\x',\y')\|_{*} \leq L_2\|\x-\x'\|^{\alpha}+L_1\|\y-\y'\|^{\alpha}.
\end{align*}
Define diameter of sets $\cX,\cY$ as \mbox{$D = \max\{\sup_{\x_1,\x_2\in\cX} \|\x_1-\x_2\|, \sup_{\y_1,\y_2\in\cY} \|\y_1-\y_2\|\}$.} Let $L=\{L_1,L_2\}$.
 Suppose both $\x$ and $\y$ players use Algorithm~\ref{alg:oftpl_cvx} to solve the minimax game. Suppose the perturbation distributions $\pertdist^1,\pertdist^2,$ used by $\x$, $\y$ players are absolutely continuous and satisfy $\Eover{\sigma\sim \pertdist^1}{\|\sigma\|_{*}} = \Eover{\sigma\sim \pertdist^2}{\|\sigma\|_{*}}=\eta$. Suppose the predictions of both the players are $\stability \eta^{-1}$-stable w.r.t $\|\cdot\|_*$.   Suppose the guesses used by $\x,\y$ players in the $t^{th}$ iteration are $\grad_{\x}f(\Tilde{\x}_{t-1},\Tilde{\y}_{t-1}), \grad_{\y}f(\Tilde{\x}_{t-1},\Tilde{\y}_{t-1})$, where  $\Tilde{\x}_{t-1},\Tilde{\y}_{t-1}$ denote the predictions of $\x,\y$ players in the $t^{th}$ iteration, if guess $g_t = 0$ was used in that iteration. Then the iterates $\{(\x_t,\y_t)\}_{t=1}^T$ generated by the OFTPL based algorithm satisfy
\begin{align*}
 \sup_{\x\in\cX,\y\in\cY}\E{f\left(\frac{1}{T}\sum_{t=1}^T\x_t,\y\right) - f\left(\x,\frac{1}{T}\sum_{t=1}^T\y_t\right)}\leq & 2L_1\left(\frac{\normcompf \normcompb D}{\sqrt{m}}\right)^{1+\alpha}+\frac{2\eta D}{T} \\
 &+ \frac{20\stability L^2}{\eta} \left(\frac{\normcompf \normcompb D}{\sqrt{m}}\right)^{2\alpha}+10L\left(\frac{5\stability L}{\eta}\right)^{\frac{1+\alpha}{1-\alpha}}
\end{align*}
\end{theorem}
\begin{proof}
Since both the players are responding to each others actions using OFTPL, using Theorem~\ref{thm:oftpl_regret}, we get the following regret bounds for the players
\begin{align*}
    \sup_{\x\in\cX} \E{\sum_{t=1}^Tf(\x_t,\y_t) - f(\x,\y_t)} &\leq L_1T\left(\frac{\normcompf \normcompb D}{\sqrt{m}}\right)^{1+\alpha} +  \eta D  \\
    &\quad+ \frac{\stability }{2\eta}\sum_{t=1}^T \E{\|\grad_{\x}f(\x_t,\y_t)-\grad_{\x}f(\Tilde{\x}_{t-1},\Tilde{\y}_{t-1})\|_*^2} \\ 
    &\quad -\frac{\eta}{2\stability }\sum_{t=1}^T \E{\|\x_t^{\infty}-\Tilde{\x}_{t-1}^{\infty}\|^2}.
\end{align*}
\begin{align*}
    \sup_{\y\in\cY} \E{\sum_{t=1}^T f(\x_t,\y)- f(\x_t,\y_t)} &\leq L_1T\left(\frac{\normcompf \normcompb D}{\sqrt{m}}\right)^{1+\alpha} +  \eta D  \\
    &\quad+ \frac{\stability }{2\eta}\sum_{t=1}^T \E{\|\grad_{\y}f(\x_t,\y_t)-\grad_{\y}f(\Tilde{\x}_{t-1},\Tilde{\y}_{t-1})\|_{*}^2} \\ 
    &\quad -\frac{\eta}{2\stability }\sum_{t=1}^T \E{\|\y_t^{\infty}-\Tilde{\y}_{t-1}^{\infty}\|^2}.
\end{align*}
First, consider the regret of the $\x$ player. Since $\|a_1+\dots +a_5\|^2 \leq 5(\|a_1\|^2\dots + \|a_5\|^2)$, we have 
\begin{align*}
    \|\grad_{\x}f(\x_t,\y_t)-\grad_{\x}f(\Tilde{\x}_{t-1},\Tilde{\y}_{t-1})\|_{*}^2 \leq &5\|\grad_{\x}f(\x_t,\y_t)-\grad_{\x}f(\x_t^{\infty},\y_t)\|_{*}^2\\
    &\quad +5\|\grad_{\x}f(\x_t^{\infty},\y_t)-\grad_{\x}f(\x_t^{\infty},\y_t^{\infty})\|_{*}^2\\
    &\quad +5\|\grad_{\x}f(\x_t^{\infty},\y_t^{\infty})-\grad_{\x}f(\Tilde{\x}_{t-1}^{\infty},\Tilde{\y}_{t-1}^{\infty})\|_{*}^2\\
    &\quad +5\|\grad_{\x}f(\Tilde{\x}_{t-1}^{\infty},\Tilde{\y}_{t-1}^{\infty})-\grad_{\x}f(\Tilde{\x}_{t-1}^{\infty},\Tilde{\y}_{t-1})\|_{*}^2\\
    &\quad +5\|\grad_{\x}f(\Tilde{\x}_{t-1}^{\infty},\Tilde{\y}_{t-1})-\grad_{\x}f(\Tilde{\x}_{t-1},\Tilde{\y}_{t-1})\|_{*}^2\\
    &\stackrel{(a)}{\leq} 5L_1^2\|\x_t-\x_t^{\infty}\|^{2\alpha}+5L_1^2\|\Tilde{\x}_{t-1}-\Tilde{\x}_{t-1}^{\infty}\|^{2\alpha}\\
    &\quad + 5L_2^2\|\y_t-\y_t^{\infty}\|^{2\alpha}+5L_2^2\|\Tilde{\y}_{t-1}-\Tilde{\y}_{t-1}^{\infty}\|^{2\alpha}\\
    &\quad + 5\|\grad_{\x}f(\x_t^{\infty},\y_t^{\infty})-\grad_{\x}f(\Tilde{\x}_{t-1}^{\infty},\Tilde{\y}_{t-1}^{\infty})\|_{*}^2.
\end{align*}
where $(a)$ follows from the Holder's smoothness of $f$. Using a similar technique as in the proof of Theorem~\ref{thm:oftpl_regret}, relying on Holders inequality, we get 
\begin{align*}
\E{\|\x_t-\x_t^{\infty}\|^{2\alpha}|\Tilde{\x}_{t-1}, \Tilde{\y}_{t-1},\x_{1:t-1},\y_{1:t-1}} &\leq \E{\|\x_t-\x_t^{\infty}\|^{2}|\Tilde{\x}_{t-1}, \Tilde{\y}_{t-1},\x_{1:t-1},\y_{1:t-1}}^{\alpha}\\
&\leq \normcompf^{2\alpha}\E{\|\x_t-\x_t^{\infty}\|^{2}_2|\Tilde{\x}_{t-1}, \Tilde{\y}_{t-1},\x_{1:t-1},\y_{1:t-1}}^{\alpha}\\
&\stackrel{(a)}{\leq} \left(\frac{\normcompf \normcompb D}{\sqrt{m}}\right)^{2\alpha},
\end{align*}
where $(a)$ follows from the fact that conditioned on past randomness, $\x_t-\x_t^{\infty}$ is the average of $m$  i.i.d bounded mean $0$ random variables, the variance of which scales as $O(D^2/m)$.
A similar bound holds for the expectation of other quantities appearing in the  RHS of the above equation. 
Using this, the regret of $\x$ player can be upper bounded as
\begin{align*}
    \sup_{\x\in\cX}\E{\sum_{t=1}^Tf(\x_t,\y_t) - f(\x,\y_t)} &\leq L_1T\left(\frac{\normcompf \normcompb D}{\sqrt{m}}\right)^{1+\alpha}+\eta D  +  \frac{10\stability L^2 T}{\eta}\left(\frac{\normcompf \normcompb D}{\sqrt{m}}\right)^{2\alpha}  \\ 
    &\quad + \frac{5\stability }{2\eta}\sum_{t=1}^T \E{\|\grad_{\x}f(\x_t^{\infty},\y_t^{\infty})-\grad_{\x}f(\Tilde{\x}_{t-1}^{\infty},\Tilde{\y}_{t-1}^{\infty})\|_{*}^2} \\
    &\quad -\frac{\eta}{2\stability }\sum_{t=1}^T \E{\|\x_t^{\infty}-\Tilde{\x}_{t-1}^{\infty}\|^2}.
\end{align*}
Similarly, the regret of $\y$ player can be bounded as
\begin{align*}
    \sup_{\y\in\cY}\E{\sum_{t=1}^T f(\x_t,\y)- f(\x_t,\y_t)} &\leq L_1T\left(\frac{\normcompf \normcompb D}{\sqrt{m}}\right)^{1+\alpha} +  \eta D  +  \frac{10\stability L^2 T}{\eta}\left(\frac{\normcompf \normcompb D}{\sqrt{m}}\right)^{2\alpha}  \\
    &\quad + \frac{5\stability }{2\eta}\sum_{t=1}^T \E{\|\grad_{\y}f(\x_t^{\infty},\y_t^{\infty})-\grad_{\y}f(\Tilde{\x}_{t-1}^{\infty},\Tilde{\y}_{t-1}^{\infty})\|_{*}^2} \\ 
    &\quad -\frac{\eta}{2\stability }\sum_{t=1}^T \E{\|\y_t^{\infty}-\Tilde{\y}_{t-1}^{\infty}\|^2}.
\end{align*}
Summing the above two inequalities, we get
\begin{align*}
    \sup_{\x\in\cX \y\in\cY}\E{\sum_{t=1}^T f(\x_t,\y)- f(\x,\y_t)} &\leq 2L_1T\left(\frac{\normcompf \normcompb D}{\sqrt{m}}\right)^{1+\alpha} +  2\eta D  +  \frac{20\stability L^2 T}{\eta}\left(\frac{\normcompf \normcompb D}{\sqrt{m}}\right)^{2\alpha}  \\
    &\quad + \frac{5\stability }{2\eta}\sum_{t=1}^T \E{\|\grad_{\x}f(\x_t^{\infty},\y_t^{\infty})-\grad_{\x}f(\Tilde{\x}_{t-1}^{\infty},\Tilde{\y}_{t-1}^{\infty})\|_{*}^2} \\ 
    &\quad + \frac{5\stability }{2\eta}\sum_{t=1}^T \E{\|\grad_{\y}f(\x_t^{\infty},\y_t^{\infty})-\grad_{\y}f(\Tilde{\x}_{t-1}^{\infty},\Tilde{\y}_{t-1}^{\infty})\|_{*}^2} \\ 
    &\quad -\frac{\eta}{2\stability }\sum_{t=1}^T \left(\E{\|\y_t^{\infty}-\Tilde{\y}_{t-1}^{\infty}\|^2}+\E{\|\x_t^{\infty}-\Tilde{\x}_{t-1}^{\infty}\|^2}\right).
\end{align*}
From Holder's smoothness assumption on $f$, we have
\begin{align*}
    \E{\|\grad_{\x}f(\x_t^{\infty},\y_t^{\infty})-\grad_{\x}f(\Tilde{\x}_{t-1}^{\infty},\Tilde{\y}_{t-1}^{\infty})\|_{*}^2} & \leq 2\E{\|\grad_{\x}f(\x_t^{\infty},\y_t^{\infty})-\grad_{\x}f(\x_t^{\infty},\Tilde{\y}_{t-1}^{\infty})\|_{*}^2}\\
    &\quad + 2\E{\|\grad_{\x}f(\x_t^{\infty},\Tilde{\y}_{t-1}^{\infty})-\grad_{\x}f(\Tilde{\x}_{t-1}^{\infty},\Tilde{\y}_{t-1}^{\infty})\|_{*}^2}\\
    &\stackrel{(a)}{\leq} 2L^2 \E{\|\x_t^{\infty}-\Tilde{\x}_{t-1}^{\infty}\|^{2\alpha}} + 2L^2\E{\|\y_t^{\infty}-\Tilde{\y}_{t-1}^{\infty}\|^{2\alpha}},
\end{align*}
Using a similar argument, we get
\begin{align*}
    \E{\|\grad_{\y}f(\x_t^{\infty},\y_t^{\infty})-\grad_{\y}f(\Tilde{\x}_{t-1}^{\infty},\Tilde{\y}_{t-1}^{\infty})\|_{*}^2} \leq 2L^2 \E{\|\x_t^{\infty}-\Tilde{\x}_{t-1}^{\infty}\|^{2\alpha}} + 2L^2\E{\|\y_t^{\infty}-\Tilde{\y}_{t-1}^{\infty}\|^{2\alpha}}.
\end{align*}
Plugging this in the previous bound, we get
\begin{align*}
    \sup_{\x\in\cX \y\in\cY}\E{\sum_{t=1}^T f(\x_t,\y)- f(\x,\y_t)} &\leq 2L_1T\left(\frac{\normcompf \normcompb D}{\sqrt{m}}\right)^{1+\alpha} +  2\eta D  +  \frac{20\stability L^2 T}{\eta}\left(\frac{\normcompf \normcompb D}{\sqrt{m}}\right)^{2\alpha}\\
    &\quad+\frac{10CL^2}{\eta}\sum_{t=1}^T \left(\E{\|\x_t^{\infty}-\Tilde{\x}_{t-1}^{\infty}\|^{2\alpha}} + \E{\|\y_t^{\infty}-\Tilde{\y}_{t-1}^{\infty}\|^{2\alpha}}\right)\\
    &\quad -\frac{\eta}{2\stability }\sum_{t=1}^T \left(\E{\|\y_t^{\infty}-\Tilde{\y}_{t-1}^{\infty}\|^2}+\E{\|\x_t^{\infty}-\Tilde{\x}_{t-1}^{\infty}\|^2}\right).
\end{align*}
\paragraph{Case $\alpha=1$.} We first consider the case of $\alpha = 1$. In this case, choosing $\eta> \sqrt{20}CL$, we get
\begin{align*}
    \sup_{\x\in\cX \y\in\cY}\E{\sum_{t=1}^T f(\x_t,\y)- f(\x,\y_t)} &\leq 2L_1T\left(\frac{\normcompf \normcompb D}{\sqrt{m}}\right)^{1+\alpha} +  2\eta D  +  \frac{20\stability L^2 T}{\eta}\left(\frac{\normcompf \normcompb D}{\sqrt{m}}\right)^{2\alpha}.
\end{align*}
\paragraph{General $\alpha$.}  The more general case relies on AM-GM inequality. Consider the following
\begin{align*}
    \frac{10CL^2}{\eta}\|\x_t^{\infty}-\Tilde{\x}_{t-1}^{\infty}\|^{2\alpha} &= \left((2\alpha C)^{\frac{\alpha}{1-\alpha}}\eta^{-\frac{1+\alpha}{1-\alpha}}(10CL^2)^{\frac{1}{1-\alpha}}\right)^{1-\alpha}\left(\frac{\|\x_t^{\infty}-\Tilde{\x}_{t-1}^{\infty}\|^2}{2\alpha C\eta^{-1}}\right)^{\alpha}\\
    &\stackrel{(a)}{\leq} (1-\alpha) \left((2\alpha C)^{\frac{\alpha}{1-\alpha}}\eta^{-\frac{1+\alpha}{1-\alpha}}(10CL^2)^{\frac{1}{1-\alpha}}\right)+ \frac{\eta}{2C}\|\x_t^{\infty}-\Tilde{\x}_{t-1}^{\infty}\|^2\\
    &=  \sqrt{20}L \left( \frac{\sqrt{20}CL}{\eta}\right)^{\frac{1+\alpha}{1-\alpha}}+ \frac{\eta}{2C}\|\x_t^{\infty}-\Tilde{\x}_{t-1}^{\infty}\|^2
\end{align*}
where $(a)$ follows from AM-GM inequality. Plugging this in the previous bound, we get
\begin{align*}
 \sup_{\x\in\cX \y\in\cY}\E{\sum_{t=1}^T f(\x_t,\y)- f(\x,\y_t)}\leq & 2L_1T\left(\frac{\normcompf \normcompb D}{\sqrt{m}}\right)^{1+\alpha}+2\eta D \\
 &+  \frac{20\stability L^2T}{\eta} \left(\frac{\normcompf \normcompb D}{\sqrt{m}}\right)^{2\alpha}+4\sqrt{5}LT\left(\frac{\sqrt{20}\stability L}{\eta}\right)^{\frac{1+\alpha}{1-\alpha}}.
\end{align*}
The claim of the theorem then follows from the observation that
\begin{align*}
  \E{f\left(\frac{1}{T}\sum_{t=1}^T\x_t,\y\right) - f\left(\x,\frac{1}{T}\sum_{t=1}^T\y_t\right)} \leq  \frac{1}{T}\E{\sum_{t=1}^T f(\x_t,\y)- f(\x,\y_t)}.
\end{align*}

\end{proof}
\subsection{Proof of Theorem~\ref{thm:oftpl_cvx_smooth_games_uniform}}
To prove the Theorem, we instantiate Theorem~\ref{thm:oftpl_cvx_smooth_games} for the uniform noise distribution. As shown in Corollary~\ref{cor:ftpl_cvx_gaussian}, the predictions of OFTPL are $dD\eta^{-1}$-stable in this case. Plugging this in the bound of Theorem~\ref{thm:oftpl_cvx_smooth_games} and using the fact that $\normcompf=\normcompb=1$ and $\alpha =1$ gives us
\begin{align*}
 \sup_{\x\in\cX,\y\in\cY}\E{f\left(\frac{1}{T}\sum_{t=1}^T\x_t,\y\right) - f\left(\x,\frac{1}{T}\sum_{t=1}^T\y_t\right)}\leq & 2L\left(\frac{D}{\sqrt{m}}\right)^{2}+\frac{2\eta D}{T} \\
 &+ \frac{20dD L^2}{\eta} \left(\frac{D}{\sqrt{m}}\right)^{2}+10L\left(\frac{5dD L}{\eta}\right)^{\infty}.
\end{align*}
Plugging in $\eta = 6dD(L+1)$, $m=T$ in the above bound gives us
\begin{align*}
 \sup_{\x\in\cX,\y\in\cY}\E{f\left(\frac{1}{T}\sum_{t=1}^T\x_t,\y\right) - f\left(\x,\frac{1}{T}\sum_{t=1}^T\y_t\right)}\leq & \order{\frac{dD^2(L+1)}{T}}.
\end{align*}
\section{Nonconvex-Nonconcave Games}\label{sec:ncvx-games}
Our algorithm for nonconvex-nonconcave games is presented in Algorithm~\ref{alg:oftpl_noncvx_games}. Note that in each iteration of this game, both the players play empirical distributions $(P_t,Q_t)$. Before presenting the proof of Theorem~\ref{thm:oftpl_noncvx_smooth_games_exp}, we first present a more general result in Section~\ref{sec:noncvx_games_general}.  Theorem~\ref{thm:oftpl_noncvx_smooth_games_exp} immediately follows from our general result by instantiating it for exponential noise distribution.
\begin{algorithm}[t]
\caption{OFTPL for nonconvex-nonconcave games}
\label{alg:oftpl_noncvx_games}
\begin{algorithmic}[1]
  \small
  \State \textbf{Input:}  Perturbation Distributions $\pertdist^1,\pertdist^2$ of $\x,\y$ players, number of samples $m,$ iterations $T$
  \For{$t = 1 \dots T$}
  \If{$t=1$}
\For{$j=1\dots m$}
  \State Sample $\sigma_{t,j}^1\sim \pertdist^1, \sigma_{t,j}^2\sim \pertdist^2$
  \State  $\x_{1,j} = \argmin_{\x\in\cX} -\sigma_{1,j}^1(\x)$
  \State  $\y_{1,j} = \argmax_{\y\in\cY} \sigma_{1,j}^2(\y)$
  \EndFor
  \State Let $P_1,Q_1 $ be the empirical distributions over $\{\x_{1,j}\}_{j=1}^m, \{\y_{1,j}\}_{j=1}^m$
\State \textbf{continue}
  \EndIf
  \State \texttt{//Compute guesses}
  \For{$j=1\dots m$}
  \State Sample $\sigma_{t,j}^1\sim \pertdist^1, \sigma_{t,j}^2\sim \pertdist^2$
  \State $\Tilde{\x}_{t-1,j}= \argmin_{\x\in\cX}\sum_{i = 1}^{t-1}f(\x,Q_i) -\sigma_{t,j}^1(\x)$
  \State $\Tilde{\y}_{t-1,j}= \argmax_{\y\in\cY}\sum_{i = 1}^{t-1}f(P_i,\y) +\sigma_{t,j}^2(\y)$
  \EndFor
  \State Let $\Tilde{P}_{t-1} $, $\Tilde{Q}_{t-1} $ be the empirical distributions over $\{\Tilde{\x}_{t-1,j}\}_{j=1}^m, \{\Tilde{\y}_{t-1,j}\}_{j=1}^m$
  \State \texttt{//Use the guesses to compute the next action}
  \For{$j=1\dots m$}
  \State Sample $\sigma_{t,j}^1\sim \pertdist^1, \sigma_{t,j}^2\sim \pertdist^2$
  \State $\x_{t,j}= \argmin_{\x\in\cX}\sum_{i = 1}^{t-1}f(\x,Q_i) + f(\x,\Tilde{Q}_{t-1}) -\sigma_{t,j}^1(\x)$
  \State $\y_{t,j}= \argmax_{\y\in\cY}\sum_{i = 1}^{t-1}f(P_i,\y) + f(\Tilde{P}_{t-1},\y) +\sigma_{t,j}^2(\y)$
  \EndFor
  \State Let $P_t, Q_t$ be the empirical distributions over $\{\x_{t,j}\}_{j=1}^m, \{\y_{t,j}\}_{j=1}^m$
  \EndFor
  \State \Return $\{(P_t, Q_t)\}_{t=1}^T$
\end{algorithmic}
\end{algorithm}

\subsection{Primal Dual Spaces}
\label{sec:primal_dual_spaces}
In this section, we present some integral probability metrics induced by popular choices of functions spaces $(\cF,\|\cdot\|_{\cF})$. 
\begin{table}[H]
\begin{center}
 \begin{tabular}{||c |c c||} 
 \hline
 $\gammaF(P, Q)$ & $\|f\|_{\cF}$&$\cF$  \\ [0.5ex] 
 \hline\hline
 Dudley Metric &  $\text{Lip}(f)+\|f\|_{\infty}$& $\{f:\text{Lip}(f) + \|f\|_{\infty} < \infty\}$ \\ 
 \hline
 \begin{tabular}[x]{@{}c@{}}Kantorovich Metric (or)\\Wasserstein-1 Metric\end{tabular}  &  $\text{Lip}(f)$ & $\{f:\text{Lip}(f) < \infty\}$\\
 \hline 
 Total Variation (TV) Distance  & $\|f\|_{\infty}$  & $\{f:\|f\|_{\infty} < \infty\}$\\ 
 \hline
 \begin{tabular}[x]{@{}c@{}} Maximum Mean Discrepancy (MMD)\\for RKHS $\mathcal{H}$\end{tabular}
  & $\|f\|_{\mathcal{H}}$ & $\{f:\|f\|_{\mathcal{H}} < \infty\}$  \\ [1ex] 
 \hline
\end{tabular}
\caption{Table showing some popular Integral Probability Metrics. Here $\text{Lip}(f)$ is the Lipschitz constant of $f$ which is defined as $\sup_{\x,\y\in\cX}|f(\x)-f(\y)|/\|\x-\y\|$ and $\|f\|_{\infty}$ is the supremum norm of $f$.}
\label{tab:ipm}
\end{center}
\end{table}

\subsection{General Result}
\label{sec:noncvx_games_general}
\begin{theorem}
\label{thm:oftpl_noncvx_smooth_games}
Consider the minimax game in Equation~\eqref{eqn:minimax_game}. Suppose the domains $\cX,\cY$ are compact subsets of $\mathbb{R}^d$.  Let $\cF,\cF'$ be the set of Lipschitz functions over $\cX,\cY$, and $\|g_1\|_{\cF},\|g_2\|_{\cF'}$ be the Lipschitz constants of functions \mbox{$g_1:\cX\to\mathbb{R}$,} $g_2:\cY\to\mathbb{R}$ w.r.t some norm $\|\cdot\|$. Suppose $f$ is such that $\max\{\sup_{\x\in\cX} \|f(\cdot,\y)\|_{\cF}, \sup_{\y\in\cY}\|f(\x,\cdot)\|_{\cF'}\}\leq G$ and satisfies the following smoothness property
\begin{align*}
    \|\grad_\x f(\x,\y)-\grad_\x f(\x',\y')\|_{*} \leq L\|\x-\x'\| + L\|\y-\y'\|,\\
    \|\grad_\y f(\x,\y)-\grad_\y f(\x',\y')\|_{*} \leq L\|\x-\x'\|+L\|\y-\y'\|.
\end{align*}
Let $\cP,\cQ$ be the set of probability distributions over $\cX,\cY$.
Define diameter of $\cP,\cQ$ as $D = \max\{\sup_{P_1,P_2\in\cP} \gammaF(P_1,P_2), \sup_{Q_1,Q_2\in\cQ} \gammaFp(Q_1,Q_2)\}$. Suppose both $\x,\y$ players use Algorithm~\ref{alg:oftpl_noncvx} to solve the game. Suppose the perturbation distributions $\pertdist^1,\pertdist^2,$ used by $\x$, $\y$ players are such that $\argmin_{\x\in\cX}f(\x)-\sigma(\x), \argmax_{\y\in\cY} f(\y)+\sigma(\y)$ have unique optimizers with probability one, for any $f$ in $\cF,\cF'$ respectively. Moreover, suppose $\Eover{\sigma\sim \pertdist^1}{\|\sigma\|_{\cF}} = \Eover{\sigma\sim \pertdist^2}{\|\sigma\|_{\cF'}}=\eta$ and predictions of both the players are $\stability \eta^{-1}$-stable w.r.t norms $\|\cdot\|_{\cF}, \|\cdot\|_{\cF'}$.  Suppose the guesses used by $\x,\y$ players in the $t^{th}$ iteration are $f(\cdot,\Tilde{Q}_{t-1}), f(\Tilde{P}_{t-1},\cdot)$, where $\Tilde{P}_{t-1},\Tilde{Q}_{t-1}$ denote the predictions of $\x,\y$ players in the $t^{th}$ iteration, if guess $g_t = 0$ was used. Then the iterates $\{(P_t,Q_t)\}_{t=1}^T$ generated by the Algorithm~\ref{alg:oftpl_cvx_games} satisfy the following, for $\eta > \sqrt{3}\stability L$
\begin{align*}
 \sup_{\x\in\cX,\y\in\cY}\E{f\left(\frac{1}{T}\sum_{t=1}^TP_t,\y\right) - f\left(\x,\frac{1}{T}\sum_{t=1}^TQ_t\right)} &= \order{\frac{\eta D}{T} + \frac{\stability D^2L^2}{\eta m}}\\
 &\quad +\order{\min\left\lbrace\frac{d\stability\normcompf^2 \normcompb^2G^2 \log(2m)}{\eta m}, \frac{CD^2L^2}{\eta}\right\rbrace}.
\end{align*}
\end{theorem}
\begin{proof}
The proof of this Theorem uses similar arguments as Theorem~\ref{thm:oftpl_cvx_smooth_games}. Since both the players are responding to each others actions using OFTPL, using Theorem~\ref{thm:oftpl_noncvx_regret}, we get the following regret bounds for the players
\begin{align*}
    \sup_{\x\in\cX} \E{\sum_{t=1}^Tf(P_t,Q_t) - f(\x,Q_t)} &\leq \eta D + \sum_{t=1}^T \frac{\stability }{2\eta}\E{\|f(\cdot,Q_t)-f(\cdot,\Tilde{Q}_{t-1})\|_{\cF}^2}\\
    &\quad -\frac{\eta}{2\stability }\sum_{t=1}^T \E{\gammaF(P_t^{\infty},\Tilde{P}_{t-1}^{\infty})^2},
\end{align*}
\begin{align*}
    \sup_{\y\in\cY} \E{\sum_{t=1}^T f(P_t,\y)- f(P_t,Q_t)} &\leq \eta D + \sum_{t=1}^T \frac{\stability }{2\eta}\E{\|f(P_t,\cdot)-f(\Tilde{P}_{t-1},\cdot)\|_{\cF'}^2}\\
    &\quad -\frac{\eta}{2\stability }\sum_{t=1}^T \E{\gammaFp(Q_t^{\infty},\Tilde{Q}_{t-1}^{\infty})^2},
\end{align*}
where $P_t^{\infty},\Tilde{P}_{t-1}^{\infty}, Q_t^{\infty},\Tilde{Q}_{t-1}^{\infty}$ are as defined in Theorem~\ref{thm:oftpl_noncvx_regret}.
First, consider the regret of the $\x$ player. We upper bound $\|f(\cdot, Q_t)-f(\cdot,\Tilde{Q}_{t-1})\|^2_{\cF}$ as
\begin{align*}
    \|f(\cdot, Q_t)-f(\cdot,\Tilde{Q}_{t-1})\|^2_{\cF} &\leq 3\|f(\cdot, Q_t)-f(\cdot,Q_{t}^{\infty})\|^2_{\cF} \\
    &\quad + 3\|f(\cdot, Q_t^{\infty})-f(\cdot,\Tilde{Q}_{t-1}^{\infty})\|^2_{\cF}\\
    &\quad + 3\|f(\cdot, \Tilde{Q}_{t-1}^{\infty})-f(\cdot,\Tilde{Q}_{t-1})\|^2_{\cF}.
\end{align*}
We now show that $\E{\|f(\cdot, Q_t)-f(\cdot,Q_{t}^{\infty})\|^2_{\cF}|\Tilde{P}_{t-1}, \Tilde{Q}_{t-1},P_{1:t-1},Q_{1:t-1}}$ is $O(1/m)$. To simplify the notation, we let \mbox{$\zeta_t = \{\Tilde{P}_{t-1}, \Tilde{Q}_{t-1},P_{1:t-1},Q_{1:t-1}\}$.} 
Let $\cN_{\epsilon}$ be the $\epsilon$-net of $\cX$ w.r.t $\|\cdot\|$. 
Then
\begin{align*}
    \|f(\cdot, Q_t)-f(\cdot,Q_{t}^{\infty})\|_{\cF} &\stackrel{(a)}{=} \sup_{\x\in\cX}\|\nabla_{\x}f(\x,Q_t) - \nabla_{\x}f(\x,Q_t^{\infty})\|_*\\
    & \stackrel{(b)}{\leq}  \sup_{\x\in\cN_{\epsilon}}\|\nabla_{\x}f(\x,Q_t) - \nabla_{\x}f(\x,Q_t^{\infty})\|_* + 2L\epsilon,
\end{align*}
where  $(a)$ follows from the definition of Lipschitz constant and $(b)$ follows from our smoothness assumption on $f$. Using this, we get
\begin{align*}
    &\E{\|f(\cdot, Q_t)-f(\cdot,Q_{t}^{\infty})\|^2_{\cF}|\zeta_t} 
     \leq 2\E{\sup_{\x\in\cN_{\epsilon}}\|\nabla_{\x}f(\x,Q_t) - \nabla_{\x}f(\x,Q_t^{\infty})\|_*^2\Big|\zeta_t}  + 8L^2\epsilon^2,
\end{align*}
Since $f$ is Lipschitz, $\|\grad_{\x}f(\x,\y)\|_*$ is bounded by $G$.  So \mbox{$\|\nabla_{\x}f(\x,Q_t) - \nabla_{\x}f(\x,Q_t^{\infty})\|_*$} is bounded by $2G$ and \mbox{$\|\nabla_{\x}f(\x,Q_t) - \nabla_{\x}f(\x,Q_t^{\infty})\|_2$} is bounded by $2\normcompf G$. Moreover, conditioned on past randomness ($\zeta_t$), $\nabla_{\x}f(\x,Q_t) - \nabla_{\x}f(\x,Q_t^{\infty})$ is a sub-Gaussian random vector and satisfies the following bound
\begin{align*}
    \E{\iprod{\u}{\nabla_{\x}f(\x,Q_t) - \nabla_{\x}f(\x,Q_t^{\infty})}|\zeta_t} \leq \exp\left(2\normcompf^2 G^2\|\u\|_2^2/m\right).
\end{align*}
From tail bounds of sub-Gaussian random vectors~\citep{hsu2012tail}, we have
\begin{align*}
    \Pr\left(\|\nabla_{\x}f(\x,Q_t) - \nabla_{\x}f(\x,Q_t^{\infty})\|_2^2 > \frac{4\normcompf^2G^2}{m}(d+2\sqrt{ds} + 2s)\Big|\zeta_t\right) \leq e^{-s},
\end{align*}
for any $s>0$. 
Using union bound, and the fact that $\log|\cN_{\epsilon}|$ is upper bounded by $d\log\left(1+2D/\epsilon\right)$, we get
\begin{align*}
    \Pr\left(\sup_{\x\in\cN_{\epsilon}}\|\nabla_{\x}f(\x,Q_t) - \nabla_{\x}f(\x,Q_t^{\infty})\|_2^2 > \frac{4\normcompf^2G^2}{m}(d+2\sqrt{ds} + 2s)\Big|\zeta_t\right) \leq e^{-s+d\log(1+2D/\epsilon)}.
\end{align*}
Let $Z = \sup_{\x\in\cN_{\epsilon}}\|\nabla_{\x}f(\x,Q_t) - \nabla_{\x}f(\x,Q_t^{\infty})\|_2^2$. 
The expectation of $Z$ can be bounded as follows
\begin{align*}
    \E{Z|\zeta_t} &= \Pr(Z \leq a|\zeta_t) \E{Z|\zeta_t, Z \leq a} + \Pr(Z > a|\zeta_t) \E{Z|\zeta_t, Z > a}\\
    &\leq a + 4\normcompf^2G^2\Pr(Z > a|\zeta_t).
\end{align*}
Choosing $\epsilon=Dm^{-1/2}, s = 3d\log(1+2m^{1/2})$, and $a= \frac{44d\normcompf^2G^2 \log(1+2m^{1/2})}{m}$, we get
\[
\E{Z|\zeta_t} \leq \frac{48d\normcompf^2G^2 \log(1+2m^{1/2})}{m}.
\]

This shows that $\E{\|f(\cdot, Q_t)-f(\cdot,Q_{t}^{\infty})\|^2_{\cF}|\zeta_t} \leq \frac{96d\normcompf^2\normcompb^2G^2 \log(1+2m^{1/2})}{m} + \frac{8D^2L^2}{m}$.
Note that another trivial upper bound for $\|f(\cdot, Q_t)-f(\cdot,Q_{t}^{\infty})\|_{\cF}$ is $DL$, which can obtained as follows
\begin{align*}
    \|f(\cdot, Q_t)-f(\cdot,Q_{t}^{\infty})\|_{\cF} &=\sup_{\x\in\cX} \|\grad_{\x}f(\x,Q_t) - \grad_{\x}f(\x,Q_t^{\infty})\|_*\\
    & = \|\Eover{\y_1\sim Q_t,\y_2\sim Q_t^{\infty}}{\grad_{\x}f(\x,\y_1) - \grad_{\x}f(\x,\y_2)}\|_*\\
    &\stackrel{(a)}{\leq} LD,
\end{align*}
where $(a)$ follows from the smoothness assumption on $f$ and the fact that the diameter of $\cX$ is $D$. 
When $L$ is close to $0$, this bound can be much better than the above bound. 
So we have
\[
\E{\|f(\cdot, Q_t)-f(\cdot,Q_{t}^{\infty})\|^2_{\cF}|\zeta_t} \leq \min\left(\frac{96d\normcompf^2\normcompb^2G^2 \log(1+2m^{1/2})}{m} + \frac{8D^2L^2}{m}, L^2D^2\right).
\]
Using this, the regret of the $\x$ player can be bounded as follows
\begin{align*}
    \sup_{\x\in\cX} \E{\sum_{t=1}^Tf(P_t,Q_t) - f(\x,Q_t)} &\leq \eta D +  \frac{24\stability D^2L^2T}{\eta m}\\
    &\quad+\min\left(\frac{288d\stability\normcompf^2\normcompb^2G^2T \log(1+2m^{1/2})}{\eta m}, \frac{3CD^2L^2T}{\eta}\right)\\
    &\quad + \sum_{t=1}^T \frac{3\stability }{2\eta}\E{\|f(\cdot,Q_t^{\infty})-f(\cdot,\Tilde{Q}_{t-1}^{\infty})\|_{\cF}^2}\\
    &\quad -\frac{\eta}{2\stability }\sum_{t=1}^T \E{\gammaF(P_t^{\infty},\Tilde{P}_{t-1}^{\infty})^2}.
\end{align*}
A similar analysis shows that the regret of $\y$ player can be bounded as
\begin{align*}
    \sup_{\y\in\cY} \E{\sum_{t=1}^T f(P_t,\y)- f(P_t,Q_t)} &\leq \eta D +  \frac{24\stability D^2L^2T}{\eta m}\\
    &\quad+\min\left(\frac{288d\stability\normcompf^2\normcompb^2G^2T \log(1+2m^{1/2})}{\eta m}, \frac{3CD^2L^2T}{\eta}\right)\\
    &\quad + \sum_{t=1}^T \frac{3\stability }{2\eta}\E{\|f(P_t^{\infty},\cdot)-f(\Tilde{P}_{t-1}^{\infty},\cdot)\|_{\cF'}^2}\\
    &\quad -\frac{\eta}{2\stability }\sum_{t=1}^T \E{\gammaFp(Q_t^{\infty},\Tilde{Q}_{t-1}^{\infty})^2},
\end{align*}
Summing the above two inequalities, we get
\begin{align*}
    \sup_{\x\in\cX,\y\in\cY} \E{\sum_{t=1}^T f(P_t,\y)- f(P,Q_t)}&\leq 2\eta D + \frac{48\stability D^2L^2T}{\eta m}\\
    &\quad+\min\left(\frac{576d\stability\normcompf^2\normcompb^2G^2T \log(1+2m^{1/2})}{\eta m}, \frac{6CD^2L^2T}{\eta}\right)\\
    &\quad + \sum_{t=1}^T \frac{3\stability }{2\eta}\E{\|f(\cdot,Q_t^{\infty})-f(\cdot,\Tilde{Q}_{t-1}^{\infty})\|_{\cF}^2}\\
    &\quad + \sum_{t=1}^T \frac{3\stability }{2\eta}\E{\|f(P_t^{\infty},\cdot)-f(\Tilde{P}_{t-1}^{\infty},\cdot)\|_{\cF'}^2}\\
    &\quad -\frac{\eta}{2\stability }\sum_{t=1}^T \left(\E{\gammaF(P_t^{\infty},\Tilde{P}_{t-1}^{\infty})^2} + \E{\gammaFp(Q_t^{\infty},\Tilde{Q}_{t-1}^{\infty})^2}\right).
\end{align*}
From our assumption on smoothness of $f$, we have
\[
\|f(\cdot,Q_t^{\infty})-f(\cdot,\Tilde{Q}_{t-1}^{\infty})\|_{\cF} \leq L \gammaFp(Q_t^{\infty},\Tilde{Q}_{t-1}^{\infty}),\quad \|f(P_t^{\infty},\cdot)-f(\Tilde{P}_{t-1}^{\infty},\cdot)\|_{\cF'} \leq L \gammaF(P_t^{\infty},\Tilde{P}_{t-1}^{\infty}).
\]
To see this, consider the following
\begin{align*}
    \|f(\cdot,Q_t^{\infty})-f(\cdot,\Tilde{Q}_{t-1}^{\infty})\|_{\cF}&=\sup_{\x\in\cX}\|\grad_{\x}f(\x,Q_t^{\infty})-\grad_{\x}f(\x,\Tilde{Q}_{t-1}^{\infty})\|_{*}\\
    &= \sup_{\x\in\cX, \|\u\|\leq 1} \iprod{\u}{\grad_{\x}f(\x,Q_t^{\infty})-\grad_{\x}f(\x,\Tilde{Q}_{t-1}^{\infty})}\\
    &= \sup_{\x\in\cX, \|\u\|\leq 1} \Eover{\y\sim Q_t^{\infty}}{\iprod{\u}{\grad_{\x}f(\x,\y)}}-\Eover{\y\sim \Tilde{Q}_{t-1}^{\infty}}{\iprod{\u}{\grad_{\x}f(\x,\y)}}\\
    &\leq \gammaFp(Q_t^{\infty}, \Tilde{Q}_{t-1}^{\infty})\sup_{\x\in\cX, \|\u\|\leq 1}\|\iprod{\u}{\grad_{\x}f(\x,\cdot)}\|_{\cF'}\\
    &= \gammaFp(Q_t^{\infty}, \Tilde{Q}_{t-1}^{\infty})\sup_{\x\in\cX, \|\u\|\leq 1}\left(\sup_{\y_1\neq \y_2 \in \cY}\frac{|\iprod{\u}{\grad_{\x}f(\x,\y_1)}-\iprod{\u}{\grad_{\x}f(\x,\y_2)}|}{\|\y_1-\y_2\|}\right)\\
    &\leq \gammaFp(Q_t^{\infty}, \Tilde{Q}_{t-1}^{\infty})\sup_{\x\in\cX}\left(\sup_{\y_1\neq \y_2 \in \cY}\frac{\|\grad_{\x}f(\x,\y_1)-\grad_{\x}f(\x,\y_2)\|_{*}}{\|\y_1-\y_2\|}\right)\\
    &\stackrel{(a)}{\leq} L  \gammaFp(Q_t^{\infty}, \Tilde{Q}_{t-1}^{\infty}),
\end{align*}
where $(a)$ follows from smoothness of $f$.
Substituting this in the previous equation, and choosing $\eta > \sqrt{3}\stability L$, we get
\begin{align*}
    \sup_{\x\in\cX,\y\in\cY} \E{\sum_{t=1}^T f(P_t,\y)- f(P,Q_t)}&\leq 2\eta D + \frac{48\stability D^2L^2T}{\eta m}\\
    &\quad+\min\left(\frac{576d\stability\normcompf^2\normcompb^2G^2T \log(1+2m^{1/2})}{\eta m}, \frac{6CD^2L^2T}{\eta}\right)\\
\end{align*}
This finishes the proof of the Theorem.
\end{proof}
\begin{remark}
We note that a similar result can be obtained for other choice of function classes such as the set of all bounded and Lipschitz functions. The only difference between proving such a result vs. proving Theorem~\ref{thm:oftpl_noncvx_smooth_games} is in bounding $\|f(\cdot, Q_t)-f(\cdot,Q_{t}^{\infty})\|_{\cF}$.
\end{remark}
\subsection{Proof of Theorem~\ref{thm:oftpl_noncvx_smooth_games_exp}}
To prove the Theorem, we instantiate Theorem~\ref{thm:oftpl_noncvx_smooth_games} for exponential noise distribution. Recall, in Corollary~\ref{cor:ftpl_noncvx_exp}, we showed that $\Eover{\sigma}{\|\sigma\|_{\cF}} = \eta\log{d}$ and OFTPL is $\order{d^2D\eta^{-1}}$ stable w.r.t $\|\cdot\|_{\cF}$, for this choice of perturbation distribution (similar results hold for $(\cF',\|\cdot\|_{\cF'})$). Substituting this in the bounds of Theorem~\ref{thm:oftpl_noncvx_smooth_games} and using the fact that $\normcompf=\sqrt{d},\normcompb = 1$, we get
\begin{align*}
 \sup_{\x\in\cX,\y\in\cY}\E{f\left(\frac{1}{T}\sum_{t=1}^TP_t,\y\right) - f\left(\x,\frac{1}{T}\sum_{t=1}^TQ_t\right)} &= \order{\frac{\eta D\log{d}}{T} + \frac{d^2 D^3L^2}{\eta m}}\\
 &\quad +\order{\min\left\lbrace\frac{d^4DG^2 \log(2m)}{\eta m}, \frac{d^2D^3L^2}{\eta}\right\rbrace}.
\end{align*}
Choosing $\eta = 10d^2D(L+1), m=T$, we get
\begin{align*}
 \sup_{\x\in\cX,\y\in\cY}\E{f\left(\frac{1}{T}\sum_{t=1}^TP_t,\y\right) - f\left(\x,\frac{1}{T}\sum_{t=1}^TQ_t\right)} &= \order{\frac{d^2 D^2(L+1)\log{d}}{T}}\\
 &\quad +\order{\min\left\lbrace\frac{d^2G^2 \log(T)}{ LT}, D^2L\right\rbrace}.
\end{align*}
\section{Choice of Perturbation Distributions}
\label{sec:pert_dist_choice}
\paragraph{Regularization of some Perturbation Distributions.}
We first study the regularization effect of various perturbation distributions. 
Table~\ref{tab:reg_linf} presents the regularizer $R$ corresponding to some commonly  used perturbation distributions, when the action space $\cX$ is $\ell_{\infty}$ ball of radius $1$ centered at origin.
\begin{table}[H]
\begin{center}
 \begin{tabular}{||c |c||} 
 \hline
 Perturbation Distribution $\pertdist$&  Regularizer \\ [0.5ex] 
 \hline\hline
 Uniform over $[0,\eta]^d$ &  $\eta \|\x-1\|_2^2$ \\ 
 \hline
 Exponential $P(\sigma > t)=\exp(-t/\eta)$ &  $\displaystyle\sum_i\eta(\x_i+1)\left[\log(\x_i+1) - (1+\log 2)\right]$ \\
 \hline 
 Gaussian $P(\sigma =t)\propto e^{-t^2/2\eta^2}$ &  $\displaystyle\sum_i \sup_{u\in\mathbb{R}} u\left[\x_i-1+2F(-u/\eta)\right]$ \\ 
 \hline
\end{tabular}
\caption{Regularizers corresponding to various perturbation distributions used in FTPL when the action space $\cX$ is $\ell_{\infty}$ ball of radius $1$ centered at origin. Here, $F$ is the CDF of a standard normal random variable.}
\label{tab:reg_linf}
\end{center}
\end{table}

\paragraph{Dimension independent rates.} Recall, the OFTPL algorithm described in Algorithm~\ref{alg:oftpl_cvx_games} converges at $\order{d/T}$ rate to a Nash equilibrium of smooth convex-concave games (see Theorem~\ref{thm:oftpl_cvx_smooth_games_uniform}). We now show that for certain constraint sets $\cX, \cY$, by choosing the perturbation distributions appropriately, the dimension dependence in the rates can \emph{potentially} be removed.

Suppose the action set is $\cX=\{\x:\|\x\|_2 \leq 1\}$. Suppose the perturbation distribution $\pertdist$ is the multivariate Gaussian distribution with mean $0$ and covariance $\eta^{2}I_{d\times d}$, where  $I_{d\times d}$ is the identity matrix.    We now try to explicitly compute the reguralizer corresponding to this perturbation distribution and action set. Define function $\Psi$ as
\[
\Psi(f) = \Eover{\sigma}{\max_{\x\in\cX} \iprod{f+\sigma}{\x}}= \Eover{\sigma}{\|f+\sigma\|_2}.
\]
As shown in Proposition~\ref{prop:ftpl_ftrl_connection}, the regularizer $R$ corresponding to any perturbation distribution  is given by the Fenchel conjugate of $\Psi$ 
\[
R(\x) = \sup_{f}\iprod{f}{\x} - \Psi(f).
\]
Since getting an exact expression for $R$ is a non-trivial task, we only compute an \emph{approximate expression} for $R$. Consider the high dimensional setting (\emph{i.e.,} very large $d$).
In this setting, $\|f+\sigma\|_2$, for $\sigma$ drawn from $\cN(0,\eta^2I_{d\times d})$, can be approximated as follows
\begin{align*}
    \|f+\sigma\|_2 &= \sqrt{\|f\|_2^2 + \|\sigma\|_2^2 + 2\iprod{f}{\sigma}}\\
    &\stackrel{(a)}{\approx} \sqrt{\|f\|_2^2 + \eta^2d + 2\iprod{f}{\sigma}}\\
    &\stackrel{(b)}{\approx} \sqrt{\|f\|_2^2 + \eta^2d}
\end{align*}
where $(a)$ follows from the fact that $\|\sigma\|_2^2$ is highly concentrated around $\eta^2d$~\citep{hsu2012tail}. To be precise
\[
\mathbb{P}(\|\sigma\|_2^2 \geq \eta^2(d + 2\sqrt{dt} + 2t)) \leq e^{-t}.
\]
A similar bound holds for the lower tail. Approximation $(b)$ follows from the fact that $\iprod{f}{\sigma}$ is a Gaussian random variable with mean $0$ and variance $\eta^2\|f\|_2^2$, and with high probability its magnitude is upper bounded by $\Tilde{O}(\eta\|f\|_2)$. Since $\eta\|f\|_2 \ll \sqrt{d}\eta\|f\|_2 \leq \|f\|_2^2 + \eta^2d$, approximation $(b)$ holds. This shows that $\Psi(f)$ can be approximated as 
\[
\Psi(f) \approx \sqrt{\|f\|_2^2 + \eta^2d}.
\]
Using this approximation, we now compute the reguralizer corresponding to the perturbation distribution
\[
R(\x) = \sup_{f}\iprod{f}{\x} - \Psi(f) \approx \sup_{f}\iprod{f}{\x} - \sqrt{\|f\|_2^2 + \eta^2d} = -\eta\sqrt{d}\sqrt{1-\|\x\|_2^2}.
\]
This shows that $R$ is $\eta\sqrt{d}$-strongly convex w.r.t $\|\cdot\|_2$ norm. Following duality between strong convexity and strong smoothness, $\Psi(f)$ is $(\eta^2d)^{-1/2}$ strongly smooth w.r.t $\|\cdot\|_2$ norm and satisfies
\[
\|\grad\Psi(f_1)-\grad\Psi(f_2)\|_2 \leq (\eta^2d)^{-1/2}\|f_1-f_2\|_2.
\]
This shows that the predictions of OFTPL are $(\eta^2d)^{-1/2}$ stable w.r.t $\|\cdot\|_2$ norm. 
We now instantiate Theorem~\ref{thm:oftpl_cvx_smooth_games} for this perturbation distribution and for constraint sets which are unit balls centered at origin, and use the above stability bound, together with the fact that $\Eover{\sigma}{\|\sigma\|_2} \approx \eta\sqrt{d}$. Suppose $f$ is smooth w.r.t $\|\cdot\|_2$ norm and satisfies
\begin{align*}
    \|\grad_\x f(\x,\y)-\grad_\x f(\x',\y')\|_{2}+ \|\grad_\y f(\x,\y)-\grad_\y f(\x',\y')\|_{2} \leq L\|\x-\x'\|_2 + L\|\y-\y'\|_2.
\end{align*}
Then Theorem~\ref{thm:oftpl_cvx_smooth_games} gives us the following rates of convergence to a NE
\begin{align*}
 \sup_{\x\in\cX,\y\in\cY}\E{f\left(\frac{1}{T}\sum_{t=1}^T\x_t,\y\right) - f\left(\x,\frac{1}{T}\sum_{t=1}^T\y_t\right)}\leq & \frac{ 2L_1}{m}+\frac{2\eta\sqrt{d}}{T} \\
 &+ \frac{20 L^2}{\eta\sqrt{d}} \left(\frac{ 1}{m}\right)+10L\left(\frac{5 L}{\eta\sqrt{d}}\right)^{\infty}
\end{align*}
Choosing $\eta = 6L/\sqrt{d}, m=T$, we get $\order{\frac{L}{T}}$ rate of convergence. Although, these rates are dimension independent, we note that our stability bound is only approximate. More accurate analysis is needed to actually claim that Algorithm~\ref{alg:oftpl_cvx_games} achieves dimension independent rates in this setting.
That being said, for general constraints sets, we believe  one can get dimension independent rates by choosing the perturbation distribution appropriately.

\section{High Probability Bounds}
\label{sec:hp_bounds}
In this section, we provide high probability bounds for Theorems~\ref{thm:oftpl_regret},~\ref{thm:oftpl_cvx_smooth_games_uniform}. Our results rely on the following concentration inequalities.
\begin{proposition}[\citet{jin2019short}]
\label{prop:azuma}
Let $X_1,\dots X_K$ be $K$ independent mean $0$ vector-valued random variables such that $\|X_i\|_2\leq B_i$. Then
\[
\mathbb{P}\left(\norm{\sum_{i=1}^KX_i}_2  \geq t\right) \leq 2\exp\left(-c\frac{t^2}{\sum_{i=1}^KB_i^2}\right),
\]
where $c>0$ is a universal constant.
\end{proposition}
We also need the following concentration inequality for martingales.
\begin{proposition}[\citet{wainwright2019high}]
\label{prop:martingale_diff}
Let $X_1,\dots X_K \in \mathbb{R}$ be a martingale difference sequence, where $\E{X_i|\cF_{i-1}} = 0$. Assume that $X_i$ satisfy the following tail condition, for some scalar $B_i>0$
\[
\mathbb{P}\left(\Big|\frac{X_i}{B_i}\Big| \geq z\Big| \cF_{i-1}\right) \leq 2\exp(-z^2).
\]
Then 
\[
\mathbb{P}\left(\Big|\sum_{i=1}^K X_i\Big| \geq z\right)\leq 2\exp\left(-c\frac{z^2}{\sum_{i=1}^KB_i^2}\right),
\]
where $c>0$ is a universal constant.
\end{proposition}
\subsection{Online Convex Learning}
In this section, we present a high probability version of Theorem~\ref{thm:oftpl_regret}.
\begin{theorem}
\label{thm:oftpl_regret_hp}
Suppose the perturbation distribution $\pertdist$ is absolutely continuous w.r.t Lebesgue measure. 
Let $D$ be the diameter of $\cX$ w.r.t $\|\cdot\|$, which is defined as $D= \sup_{\x_1,\x_2\in\cX} \|\x_1-\x_2\|.$ 
Let \mbox{$\eta=\Eover{\sigma}{\|\sigma\|_*},$} and suppose the predictions of OFTPL are $\stability \eta^{-1}$-stable w.r.t $\|\cdot\|_*$, where $\stability $ is a constant that depends on the set $\mathcal{X}.$
Suppose, the sequence of loss functions $\{f_t\}_{t=1}^T$ are $G$-Lipschitz w.r.t $\|\cdot\|$ and satisfy $\sup_{\x \in \mathcal{X}} \|\grad f_t(\x)\|_* \leq G$.
Moreover, suppose $\{f_t\}_{t=1}^T$ are Holder smooth and satisfy
\begin{align*}
    \forall \x_1,\x_2\in\cX\quad \|\grad f_t(\x_1)-\grad f_t(\x_2)\|_* \leq L\|\x_1-\x_2\|^{\alpha},
\end{align*}
for some constant $\alpha \in [0,1]$.
Then the regret of Algorithm~\ref{alg:oftpl_cvx} satisfies the following with probability at least $1-\delta$
\begin{align*}
\sup_{\x\in\cX}\sum_{t=1}^Tf_t(\x_t) - f_t(\x) 
&\leq   \eta D  + \sum_{t=1}^T\frac{\stability}{2\eta} \|\grad_t-g_{t}\|_{*}^2-\sum_{t=1}^T \frac{\eta}{2\stability }\|\x_t^{\infty}-\Tilde{\x}_{t-1}^{\infty}\|^2\\
&\quad+cGD\sqrt{\frac{T\log{2/\delta}}{m}} +  cLT\left(\frac{\normcompf^2 \normcompb^2 D^2\log{4T/\delta}}{m}\right)^{\frac{1+\alpha}{2}},
\end{align*}
where $c$ is a universal constant, $\x_t^{\infty} = \E{\x_t|g_t,f_{1:t-1}, \x_{1:t-1}}$ and $\Tilde{\x}_{t-1}^{\infty} = \E{\Tilde{\x}_{t-1}|f_{1:t-1}, \x_{1:t-1}}$ and $\Tilde{\x}_{t-1}$ denotes the prediction in the $t^{th}$ iteration of Algorithm~\ref{alg:oftpl_cvx}, if guess $g_{t}=0$ was used. Here, $\normcompf, \normcompb$ denote the norm compatibility constants of $\|\cdot\|.$
\end{theorem}
\begin{proof}
Our proof uses the same notation and similar arguments as in the proof Theorem~\ref{thm:oftpl_regret}. Recall, in Theorem~\ref{thm:oftpl_regret} we showed that the regret of OFTPL is upper bounded by 
\begin{align*}
\sum_{t=1}^Tf_t(\x_t) - f_t(\x) &\leq \sum_{t=1}^T\iprod{\x_t-\x_t^{\infty}}{\grad_t} +  \eta D + \sum_{t=1}^T \|\x_t^{\infty}-\Tilde{\x}_t^{\infty}\|\|\grad_t-g_{t}\|_{*}\\
&\quad -\frac{\eta}{2\stability }\sum_{t=1}^T\left(\|\Tilde{\x}_t^{\infty}-\x_t^{\infty}\|^2 + \|\x_t^{\infty}-\Tilde{\x}_{t-1}^{\infty}\|^2\right)\\
&\leq \sum_{t=1}^T\iprod{\x_t-\x_t^{\infty}}{\grad_t} +  \eta D + \sum_{t=1}^T\frac{\stability}{2\eta} \|\grad_t-g_{t}\|_{*}^2-\sum_{t=1}^T \frac{\eta}{2\stability }\|\x_t^{\infty}-\Tilde{\x}_{t-1}^{\infty}\|^2.
\end{align*}
From Holder's smoothness assumption, we have
\[
\iprod{\x_t-\x_t^{\infty}}{\grad_t - \grad f_t(\x_t^{\infty})} \leq L\|\x_t-\x_t^{\infty}\|^{1+\alpha}.
\]
Substituting this in the previous bound gives us
\begin{align*}
\sum_{t=1}^Tf_t(\x_t) - f_t(\x) 
&\leq  \underbrace{\sum_{t=1}^T\iprod{\x_t-\x_t^{\infty}}{\grad f_t(\x_t^{\infty})}}_{S_1}+\sum_{t=1}^TL\underbrace{\|\x_t-\x_t^{\infty}\|^{1+\alpha}}_{S_2} +  \eta D \\
&\quad + \sum_{t=1}^T\frac{\stability}{2\eta} \|\grad_t-g_{t}\|_{*}^2-\sum_{t=1}^T \frac{\eta}{2\stability }\|\x_t^{\infty}-\Tilde{\x}_{t-1}^{\infty}\|^2.
\end{align*}
We now provide high probability bounds for $S_1$ and $S_2$.
\paragraph{Bounding $S_1$.} Let $\xi_i = \{g_{i+1}, f_{i+1}, \x_{i}\}$ and let $\xi_{0:t}$ denote the union of sets $\xi_0,\xi_1,\dots, \xi_t$. Let $\zeta_t = \iprod{\x_t-\x_t^{\infty}}{\grad f_t(\x_t^{\infty})}$ with $\zeta_0=0$. Note that $\{\zeta_t\}_{t=0}^T$ is a martingale difference sequence w.r.t $\xi_{0:T}$. This is because $\E{\x_t|\xi_{0:t-1}} = \x_t^{\infty}$ and $\grad f_t(\x_t^{\infty})$ is a deterministic quantity conditioned on $\xi_{0:t-1}$. As a result $\E{\zeta_t|\xi_{0:t-1}}=0$. Moreover, conditioned on $\xi_{0:t-1}$, $\zeta_t$ is the average of $m$ independent mean $0$ random variables, each of which is bounded by $GD$. Using Proposition~\ref{prop:azuma}, we get
\[
\mathbb{P}\left(|\zeta_t| \geq s\Big| \xi_{0:t-1}\right) \leq 2\exp\left(-\frac{ms^2}{G^2D^2}\right).
\]
Using Proposition~\ref{prop:martingale_diff} on the martingale difference sequence $\{\zeta_t\}_{t=0}^T$, we get
\[
\mathbb{P}\left(\Big|\sum_{t=1}^T\zeta_t\Big| \geq s\right)\leq 2\exp\left(-c\frac{ms^2}{G^2D^2T}\right),
\]
where $c>0$ is a universal constant. 
This shows that with probability at least $1-\delta/2$, $S_1$ is upper bounded by
$ \order{\sqrt{\frac{G^2D^2T\log{\frac{2}{\delta}}}{m}}}.$
\paragraph{Bounding $S_2$.} Conditioned on $\{g_t,f_{1:t-1}, \x_{1:t-1}\}$, $\x_t-\x_t^{\infty}$ is the average of $m$ independent mean $0$ random variables which are bounded by $D$ in $\|\cdot\|$ norm. From our definition of norm compatibility constant $\normcompb$, this implies the random variables are bounded by $\normcompb D$ in $\|\cdot\|_2$.  Using Proposition~\ref{prop:azuma}, we get
\[
\mathbb{P}\left(\|\x_t-\x_t^{\infty}\|_2 \geq  \normcompb D\sqrt{\frac{c\log{4T/\delta}}{m}}\Bigg| g_t,f_{1:t-1}, \x_{1:t-1}\right) \leq  \frac{\delta}{2T}.
\]
Since the above bound holds for any set of $\{g_t,f_{1:t}, \x_{1:t-1}\}$, the same tail bound also holds without the conditioning. This shows that
\[
\mathbb{P}\left(\|\x_t-\x_t^{\infty}\|^{1+\alpha} \geq  \left(\frac{c\normcompf^2 \normcompb^2 D^2\log{4T/\delta}}{m}\right)^{\frac{1+\alpha}{2}}\right) \leq  \frac{\delta}{2T},
\]
where we converted back to $\|\cdot\|$ by introducing the norm compatibility constant $\normcompf$.
\paragraph{Bounding the regret.} Plugging the above high probability bounds for $S_1,S_2$ in the previous regret bound and using union bound, we get the following regret bound which holds with probability at least $1-\delta$
\begin{align*}
\sum_{t=1}^Tf_t(\x_t) - f_t(\x) 
&\leq  cGD\sqrt{\frac{T\log{2/\delta}}{m}}  +  cLT\left(\frac{\normcompf^2 \normcompb^2 D^2\log{4T/\delta}}{m}\right)^{\frac{1+\alpha}{2}} + \eta D \\
&\quad + \sum_{t=1}^T\frac{\stability}{2\eta} \|\grad_t-g_{t}\|_{*}^2-\sum_{t=1}^T \frac{\eta}{2\stability }\|\x_t^{\infty}-\Tilde{\x}_{t-1}^{\infty}\|^2,
\end{align*}
where $c>0$ is a universal constant.
\end{proof}

\subsection{Convex-Concave Games}
In this section, we present a high probability version of Theorem~\ref{thm:oftpl_cvx_smooth_games_uniform}.

\begin{theorem}
\label{thm:oftpl_cvx_smooth_games_uniform_hp}
Consider the minimax game in Equation~\eqref{eqn:minimax_game}. Suppose both the domains $\cX,\cY$ are compact subsets of $\mathbb{R}^d$, with diameter \mbox{$D = \max\{\sup_{\x_1,\x_2\in\cX} \|\x_1-\x_2\|_2, \sup_{\y_1,\y_2\in\cY} \|\y_1-\y_2\|_2\}$.} Suppose $f$ is convex in $\x$, concave in $\y$ and is Lipschitz w.r.t $\|\cdot\|_2$ and satisfies 
\begin{align*}
\max\left\lbrace\sup_{\x\in\cX, \y\in\cY} \|\grad_{\x}f(\x,\y)\|_{2}, \sup_{\x\in\cX,\y\in\cY}\|\grad_{\y}f(\x,\y)\|_{2}\right\rbrace\leq G.
\end{align*}
Moreover, suppose $f$ is  smooth w.r.t $\|\cdot\|_2$
\begin{align*}
    \|\grad_\x f(\x,\y)-\grad_\x f(\x',\y')\|_{2}+ \|\grad_\y f(\x,\y)-\grad_\y f(\x',\y')\|_{2} \leq L\|\x-\x'\|_2 + L\|\y-\y'\|_2.
\end{align*}
Suppose Algorithm~\ref{alg:oftpl_cvx_games} is used to solve the minimax game. Suppose the perturbation distributions used by both the players are the same and equal to the uniform distribution over $\{\x:\|\x\|_2 \leq (1+d^{-1})\eta\}.$  Suppose the guesses used by $\x,\y$ players in the $t^{th}$ iteration are $\grad_{\x}f(\Tilde{\x}_{t-1},\Tilde{\y}_{t-1}), \grad_{\y}f(\Tilde{\x}_{t-1},\Tilde{\y}_{t-1})$, where  $\Tilde{\x}_{t-1},\Tilde{\y}_{t-1}$ denote the predictions of $\x,\y$ players in the $t^{th}$ iteration, if guess $g_t = 0$ was used. If Algorithm~\ref{alg:oftpl_cvx_games} is run with $\eta = 6dD(L+1), m = T$, then the iterates $\{(\x_t,\y_t)\}_{t=1}^T$ satisfy the following bound with probability at least $1-\delta$
\begin{align*}
 \sup_{\x\in\cX,\y\in\cY}\left[f\left(\frac{1}{T}\sum_{t=1}^T\x_t,\y\right) - f\left(\x,\frac{1}{T}\sum_{t=1}^T\y_t\right)\right]=  \order{\frac{GD\sqrt{\log{\frac{8}{\delta}}}}{T}+ \frac{D^2(L+1)\left(d + \log{\frac{16T}{\delta}}\right)}{T}}.
\end{align*}
\end{theorem}
\begin{proof}
We use the same notation and proof technique as Theorems~\ref{thm:oftpl_cvx_smooth_games},~\ref{thm:oftpl_cvx_smooth_games_uniform}. From Theorem~\ref{cor:ftpl_cvx_gaussian} we know that the predictions of OFTPL are $dD\eta^{-1}$ stable w.r.t $\|\cdot\|_2$, for the particular perturbation distribution we consider here. We use this stability bound in our proof. From Theorem~\ref{thm:oftpl_regret_hp}, we have  the following regret bound for both the players, which holds with probability at least $1-\delta/2$
\begin{align*}
    \sup_{\x\in\cX} \left[\sum_{t=1}^Tf(\x_t,\y_t) - f(\x,\y_t)\right] &\leq cGD\sqrt{\frac{T\log{8/\delta}}{m}} +  cLT\left(\frac{D^2\log{16T/\delta}}{m}\right) +  \eta D  \\
    &\quad+ \frac{dD }{2\eta}\sum_{t=1}^T \left[\|\grad_{\x}f(\x_t,\y_t)-\grad_{\x}f(\Tilde{\x}_{t-1},\Tilde{\y}_{t-1})\|_2^2\right] \\ 
    &\quad -\frac{\eta}{2dD }\sum_{t=1}^T \left[\|\x_t^{\infty}-\Tilde{\x}_{t-1}^{\infty}\|_2^2\right].
\end{align*}
\begin{align*}
    \sup_{\y\in\cY} \left[\sum_{t=1}^T f(\x_t,\y)- f(\x_t,\y_t)\right] &\leq cGD\sqrt{\frac{T\log{8/\delta}}{m}} +  cLT\left(\frac{D^2\log{16T/\delta}}{m}\right) +  \eta D  \\
    &\quad+ \frac{dD }{2\eta}\sum_{t=1}^T \left[\|\grad_{\y}f(\x_t,\y_t)-\grad_{\y}f(\Tilde{\x}_{t-1},\Tilde{\y}_{t-1})\|_{2}^2\right] \\ 
    &\quad -\frac{\eta}{2dD }\sum_{t=1}^T \left[\|\y_t^{\infty}-\Tilde{\y}_{t-1}^{\infty}\|_2^2\right].
\end{align*}
First, consider the regret of the $\x$ player. From the proof of Theorem~\ref{thm:oftpl_cvx_smooth_games}, we have 
\begin{align*}
    \|\grad_{\x}f(\x_t,\y_t)-\grad_{\x}f(\Tilde{\x}_{t-1},\Tilde{\y}_{t-1})\|_{2}^2 &\leq 5L^2\|\x_t-\x_t^{\infty}\|_2^{2}+5L^2\|\Tilde{\x}_{t-1}-\Tilde{\x}_{t-1}^{\infty}\|_2^{2}\\
    &\quad + 5L^2\|\y_t-\y_t^{\infty}\|_2^{2}+5L^2\|\Tilde{\y}_{t-1}-\Tilde{\y}_{t-1}^{\infty}\|_2^{2}\\
    &\quad + 5\|\grad_{\x}f(\x_t^{\infty},\y_t^{\infty})-\grad_{\x}f(\Tilde{\x}_{t-1}^{\infty},\Tilde{\y}_{t-1}^{\infty})\|_{2}^2.
\end{align*}
Moreover, from the proof of Theorem~\ref{thm:oftpl_regret_hp}, we know that $\|\x_t-\x_t^{\infty}\|_2^{2}$ satisfies the following tail bound 
\[
\mathbb{P}\left(\|\x_t-\x_t^{\infty}\|_2^{2} \geq  \frac{c D^2\log{16T/\delta}}{m}\right) \leq  \frac{\delta}{8T}.
\]
Similar bounds hold for the quantities appearing in the regret bound of $\y$ player. Plugging this in the previous regret bounds, we get the following which hold with probability at least $1-\delta$
\begin{align*}
    \sup_{\x\in\cX} \left[\sum_{t=1}^Tf(\x_t,\y_t) - f(\x,\y_t)\right] &\leq cGD\sqrt{\frac{T\log{8/\delta}}{m}} +  \left(L+\frac{10dDL^2}{\eta}\right)\left(\frac{cD^2\log{16T/\delta}}{m}\right)T  \\
    &\quad +  \eta D+ \frac{5dD }{2\eta}\sum_{t=1}^T \left[\|\grad_{\x}f(\x_t^{\infty},\y_t^{\infty})-\grad_{\x}f(\Tilde{\x}_{t-1}^{\infty},\Tilde{\y}_{t-1}^{\infty})\|_2^2\right] \\ 
    &\quad -\frac{\eta}{2dD }\sum_{t=1}^T \left[\|\x_t^{\infty}-\Tilde{\x}_{t-1}^{\infty}\|_2^2\right].
\end{align*}
\begin{align*}
    \sup_{\y\in\cY} \left[\sum_{t=1}^T f(\x_t,\y)- f(\x_t,\y_t)\right] &\leq cGD\sqrt{\frac{T\log{8/\delta}}{m}} +  \left(L+\frac{10dDL^2}{\eta}\right)\left(\frac{cD^2\log{16T/\delta}}{m}\right)T  \\
    &\quad+  \eta D + \frac{5dD }{2\eta}\sum_{t=1}^T \left[\|\grad_{\y}f(\x_t^{\infty},\y_t^{\infty})-\grad_{\y}f(\Tilde{\x}_{t-1}^{\infty},\Tilde{\y}_{t-1}^{\infty})\|_{2}^2\right] \\ 
    &\quad -\frac{\eta}{2dD }\sum_{t=1}^T \left[\|\y_t^{\infty}-\Tilde{\y}_{t-1}^{\infty}\|_2^2\right].
\end{align*}
Summing these two regret bounds, we get
\begin{align*}
    \sup_{\x\in\cX,\y\in\cY} \left[\sum_{t=1}^Tf(\x_t,\y) - f(\x,\y_t)\right] &\leq 2cGD\sqrt{\frac{T\log{8/\delta}}{m}} +  \left(L+\frac{10dDL^2}{\eta}\right)\left(\frac{2cD^2\log{16T/\delta}}{m}\right)T +  2\eta D  \\
    &\quad+ \frac{10dD }{2\eta}\sum_{t=1}^T \left[\|\grad_{\x}f(\x_t^{\infty},\y_t^{\infty})-\grad_{\x}f(\Tilde{\x}_{t-1}^{\infty},\Tilde{\y}_{t-1}^{\infty})\|_2^2\right] \\ 
    &\quad+ \frac{10dD }{2\eta}\sum_{t=1}^T \left[\|\grad_{\y}f(\x_t^{\infty},\y_t^{\infty})-\grad_{\y}f(\Tilde{\x}_{t-1}^{\infty},\Tilde{\y}_{t-1}^{\infty})\|_{2}^2\right] \\ 
    &\quad -\frac{\eta}{2dD }\sum_{t=1}^T \left[\|\x_t^{\infty}-\Tilde{\x}_{t-1}^{\infty}\|_2^2+\|\y_t^{\infty}-\Tilde{\y}_{t-1}^{\infty}\|_2^2\right].
\end{align*}
From Holder's smoothness assumption on $f$, we have
\begin{align*}
    \|\grad_{\x}f(\x_t^{\infty},\y_t^{\infty})-\grad_{\x}f(\Tilde{\x}_{t-1}^{\infty},\Tilde{\y}_{t-1}^{\infty})\|_{2}^2 & \leq 2\|\grad_{\x}f(\x_t^{\infty},\y_t^{\infty})-\grad_{\x}f(\x_t^{\infty},\Tilde{\y}_{t-1}^{\infty})\|_{2}^2\\
    &\quad + 2\|\grad_{\x}f(\x_t^{\infty},\Tilde{\y}_{t-1}^{\infty})-\grad_{\x}f(\Tilde{\x}_{t-1}^{\infty},\Tilde{\y}_{t-1}^{\infty})\|_{2}^2\\
    &\leq 2L^2 \|\x_t^{\infty}-\Tilde{\x}_{t-1}^{\infty}\|_2^{2} + 2L^2\|\y_t^{\infty}-\Tilde{\y}_{t-1}^{\infty}\|_2^{2},
\end{align*}
Using a similar argument, we get
\begin{align*}
    \|\grad_{\y}f(\x_t^{\infty},\y_t^{\infty})-\grad_{\y}f(\Tilde{\x}_{t-1}^{\infty},\Tilde{\y}_{t-1}^{\infty})\|_{2}^2 \leq 2L^2 \|\x_t^{\infty}-\Tilde{\x}_{t-1}^{\infty}\|_2^{2} + 2L^2\|\y_t^{\infty}-\Tilde{\y}_{t-1}^{\infty}\|_2^{2}.
\end{align*}
Plugging this in the previous bound, and setting $\eta = 6d D(L+1), m=T$, we get the following bound which holds with probability at least $1-\delta$
\begin{align*}
    \sup_{\x\in\cX,\y\in\cY} \left[\sum_{t=1}^Tf(\x_t,\y) - f(\x,\y_t)\right] &\leq \order{GD\sqrt{\log{\frac{8}{\delta}}}+ D^2(L+1)\left(d + \log{\frac{16T}{\delta}}\right)}.
\end{align*}
\end{proof}
\subsection{Nonconvex-Nonconcave Games}
In this section, we present a high probability version of Theorem~\ref{thm:oftpl_noncvx_smooth_games_exp}.
\begin{theorem}
\label{thm:oftpl_noncvx_smooth_games_exp_hp}
Consider the minimax game in Equation~\eqref{eqn:minimax_game}. Suppose the domains $\cX,\cY$ are compact subsets of $\mathbb{R}^d$ with diameter $D = \max\{\sup_{\x_1,\x_2\in\cX} \|\x_1-\x_2\|_1, \sup_{\y_1,\y_2\in\cY} \|\y_1-\y_2\|_1\}$.  Suppose $f$ is Lipschitz w.r.t $\|\cdot\|_1$ and satisfies 
\begin{align*}
\max\left\lbrace\sup_{\x\in\cX, \y\in\cY} \|\grad_{\x}f(\x,\y)\|_{\infty}, \sup_{\x\in\cX,\y\in\cY}\|\grad_{\y}f(\x,\y)\|_{\infty}\right\rbrace\leq G.
\end{align*}
Moreover, suppose $f$ satisfies the following smoothness property
\begin{align*}
    \|\grad_\x f(\x,\y)-\grad_\x f(\x',\y')\|_{\infty} + \|\grad_\y f(\x,\y)-\grad_\y f(\x',\y')\|_{\infty} \leq L\|\x-\x'\|_1 + L\|\y-\y'\|_1.
\end{align*}
Suppose both $\x$ and $\y$ players use Algorithm~\ref{alg:oftpl_noncvx_games} to solve the game with linear perturbation functions $\sigma(\z)=\iprod{\bar{\sigma}}{\z}$, where $\bar{\sigma} \in \mathbb{R}^d$ is such that each of its entries is sampled independently from $\text{Exp}(\eta)$.  Suppose the guesses used by $\x$ and $\y$ players in the $t^{th}$ iteration are $f(\cdot,\Tilde{Q}_{t-1}), f(\Tilde{P}_{t-1},\cdot)$, where $\Tilde{P}_{t-1},\Tilde{Q}_{t-1}$ denote the predictions of $\x,\y$ players in the $t^{th}$ iteration, if guess $g_t = 0$ was used. If Algorithm~\ref{alg:oftpl_noncvx_games} is run with $\eta = 10d^2D(L+1), m = T$, then the iterates $\{(P_t,Q_t)\}_{t=1}^T$ satisfy the following with probability at least $1-\delta$
\begin{align*}
 \sup_{\x\in\cX,\y\in\cY}\sum_{t=1}^Tf(P_t,\y) - f(\x,Q_t)& = \order{\frac{d^2D^2(L+1)\log{d}}{T} + \frac{GD}{T}\sqrt{\log{\frac{8}{\delta}}}}\\
 &\quad + \order{\min\left\lbrace D^2L, \frac{d^2G^2\log{T}+dG^2\log{\frac{8}{\delta}}}{LT}\right\rbrace}.
\end{align*}
\end{theorem}
\begin{proof}
We use the same notation used in the proofs of Theorems~\ref{thm:oftpl_noncvx_regret},~\ref{thm:oftpl_noncvx_smooth_games}. Let $\cF,\cF'$ be the set of Lipschitz functions over $\cX,\cY$, and $\|g_1\|_{\cF},\|g_2\|_{\cF'}$ be the Lipschitz constants of functions \mbox{$g_1:\cX\to\mathbb{R}$,} \mbox{$g_2:\cY\to\mathbb{R}$} w.r.t $\|\cdot\|_1$. Recall, in Corollary~\ref{cor:ftpl_noncvx_exp} we showed that for our choice of perturbation distribution, $\Eover{\sigma}{\|\sigma\|_{\cF}} = \eta \log{d}$ and OFTPL is $\order{d^2D\eta^{-1}}$ stable. We use this in our proof. 

 From Theorem~\ref{thm:oftpl_noncvx_regret}, we know that the regret of $\x,\y$ players satisfy
\begin{align*}
    \sum_{t=1}^Tf(P_t,Q_t) - f(\x,Q_t) &\leq  \eta D\log{d} + \underbrace{\sum_{t=1}^T\iprod{P_t-P_t^{\infty}}{f(\cdot, Q_t)}}_{S_1} \\
    &\quad + \sum_{t=1}^T\frac{cd^2D}{2\eta}\underbrace{\|f(\cdot,Q_t)-f(\cdot,\Tilde{Q}_{t-1})\|_{\cF}^2}_{S_2}\\
    &\quad- \sum_{t=1}^T\frac{\eta}{2cd^2D}\gammaF(P_t^{\infty},\Tilde{P}_{t-1}^{\infty})^2
\end{align*}
\begin{align*}
    \sum_{t=1}^T f(P_t,\y)- f(P_t,Q_t) &\leq  \eta D\log{d}+ \sum_{t=1}^T\iprod{Q_t-Q_t^{\infty}}{f(P_t,\cdot)}\\ &\quad +\sum_{t=1}^T\frac{cd^2D}{2\eta}\|f(P_t,\cdot)-f(\Tilde{P}_{t-1},\cdot)\|_{\cF'}^2\\
    &\quad- \sum_{t=1}^T\frac{\eta}{2cd^2D}\gammaFp(Q_t^{\infty},\Tilde{Q}_{t-1}^{\infty})^2,
\end{align*}
where $c>0$ is a positive constant. 
We now provide high probability bounds for $S_1,S_2$. 
\paragraph{Bounding $S_1$.} Let $\xi_i = \{\Tilde{P}_{i}, \Tilde{Q}_i, P_{i},  Q_{i+1}\}$ with $\xi_0=\{Q_1\}$ and let $\xi_{0:t}$ denote the union of sets $\xi_0,\dots, \xi_t$. Let $\zeta_t = \iprod{P_t-P_t^{\infty}}{f(\cdot, Q_t)}$ with $\zeta_0 = 0$. Note that $\{\zeta_t\}_{t=0}^T$ is a martingale difference sequence w.r.t $\xi_{0:T}$. This is because $\E{P_t|\xi_{0:t-1}} = P_t^{\infty}$ and $f(\cdot,Q_t)$ is a deterministic quantity conditioned on $\xi_{0:t-1}$. As a result $\E{\zeta_t|\xi_{0:t-1}}=0$. Moreover, conditioned on $\xi_{0:t-1}$, $\zeta_t$ is the average of $m$ independent mean $0$ random variables, each of which is bounded by $2GD$. Using Proposition~\ref{prop:azuma}, we get
\[
\mathbb{P}\left(|\zeta_t| \geq s\Big| \xi_{0:t-1}\right) \leq 2\exp\left(-\frac{ms^2}{4G^2D^2}\right).
\]
Using Proposition~\ref{prop:martingale_diff} on the martingale difference sequence $\{\zeta_t\}_{t=0}^T$, we get
\[
\mathbb{P}\left(\Big|\sum_{t=1}^T\zeta_t\Big| \geq s\right)\leq 2\exp\left(-c\frac{ms^2}{G^2D^2T}\right),
\]
where $c>0$ is a universal constant. 
This shows that with probability at least $1-\delta/8$, $S_1$ is upper bounded by
$ \order{\sqrt{\frac{G^2D^2T\log{\frac{8}{\delta}}}{m}}}.$
\paragraph{Bounding $S_2$.} We upper bound $S_2$ as
\begin{align*}
    \|f(\cdot, Q_t)-f(\cdot,\Tilde{Q}_{t-1})\|^2_{\cF} &\leq 3\|f(\cdot, Q_t)-f(\cdot,Q_{t}^{\infty})\|^2_{\cF} \\
    &\quad + 3\|f(\cdot, Q_t^{\infty})-f(\cdot,\Tilde{Q}_{t-1}^{\infty})\|^2_{\cF}\\
    &\quad + 3\|f(\cdot, \Tilde{Q}_{t-1}^{\infty})-f(\cdot,\Tilde{Q}_{t-1})\|^2_{\cF}.
\end{align*}
We first provide a high probability bound for $\|f(\cdot, Q_t)-f(\cdot,Q_{t}^{\infty})\|^2_{\cF}$. A trivial bound for this quantity is $L^2D^2$, which can be obtained as follows
\begin{align*}
    \|f(\cdot, Q_t)-f(\cdot,Q_{t}^{\infty})\|_{\cF} &=\sup_{\x\in\cX} \|\grad_{\x}f(\x,Q_t) - \grad_{\x}f(\x,Q_t^{\infty})\|_{\infty}\\
    & = \|\Eover{\y_1\sim Q_t,\y_2\sim Q_t^{\infty}}{\grad_{\x}f(\x,\y_1) - \grad_{\x}f(\x,\y_2)}\|_{\infty}\\
    &\stackrel{(a)}{\leq} LD,
\end{align*}
where $(a)$ follows from the smoothness assumption on $f$ and the fact that the diameter of $\cX$ is $D$. 
A better bound for this quantity can be obtained as follows.
From proof of Theorem~\ref{thm:oftpl_noncvx_smooth_games}, we have
\begin{align*}
    &\|f(\cdot, Q_t)-f(\cdot,Q_{t}^{\infty})\|^2_{\cF} 
     \leq 2\sup_{\x\in\cN_{\epsilon}}\|\nabla_{\x}f(\x,Q_t) - \nabla_{\x}f(\x,Q_t^{\infty})\|_{\infty}^2  + 8L^2\epsilon^2.
\end{align*}
where   $\cN_{\epsilon}$ be the $\epsilon$-net of $\cX$ w.r.t $\|\cdot\|$. 
Recall, in the proof of Theorem~\ref{thm:oftpl_noncvx_smooth_games}, we showed the following high probability bound for the RHS quantity
\begin{align*}
    \Pr\left(\sup_{\x\in\cN_{\epsilon}}\|\nabla_{\x}f(\x,Q_t) - \nabla_{\x}f(\x,Q_t^{\infty})\|_2^2 > \frac{4dG^2}{m}(d+2\sqrt{ds} + 2s)\right) \leq e^{-s+d\log(1+2D/\epsilon)}.
\end{align*}
Choosing $\epsilon=Dm^{-1/2}, s = \log{\frac{8}{\delta}}+d\log(1+2m^{1/2})$, we get the following bound for $\sup_{\x\in\cN_{\epsilon}}\|\nabla_{\x}f(\x,Q_t) - \nabla_{\x}f(\x,Q_t^{\infty})\|_2^2$ which holds with probability at least $1-\delta/8$
\[
\sup_{\x\in\cN_{\epsilon}}\|\nabla_{\x}f(\x,Q_t) - \nabla_{\x}f(\x,Q_t^{\infty})\|_2^2 \leq \frac{20dG^2}{m}\left(\log{\frac{8}{\delta}}+d\log(1+2m^{1/2})\right).
\]
Together with our trivial bound of $D^2L^2$, this gives us the following bound for $\|f(\cdot, Q_t)-f(\cdot,Q_{t}^{\infty})\|^2_{\cF} $, which holds with probability at least $1-\delta/8$
\[
\|f(\cdot, Q_t)-f(\cdot,Q_{t}^{\infty})\|^2_{\cF}  \leq \min\left(\frac{20dG^2}{m}\left(\log{\frac{8}{\delta}}+d\log(1+2m^{1/2})\right), D^2L^2\right) + \frac{8D^2L^2}{m}.
\]
Next, we bound $\|f(\cdot, Q_t^{\infty})-f(\cdot,\Tilde{Q}_{t-1}^{\infty})\|^2_{\cF}$. From our smoothness assumption on $f$, we have
\[
\|f(\cdot,Q_t^{\infty})-f(\cdot,\Tilde{Q}_{t-1}^{\infty})\|_{\cF} \leq L \gammaFp(Q_t^{\infty},\Tilde{Q}_{t-1}^{\infty}).
\]
Combining the previous two results, we get the following upper bound for $S_2$ which holds with probability at least $1-\delta/8$
\begin{align*}
    \|f(\cdot, Q_t)-f(\cdot,\Tilde{Q}_{t-1})\|^2_{\cF} &\leq 3L^2\gammaFp(Q_t^{\infty},\Tilde{Q}_{t-1}^{\infty})^2  + \frac{48D^2L^2}{m} \\
    &\quad + \min\left(\frac{120dG^2}{m}\left(\log{\frac{8}{\delta}}+d\log(1+2m^{1/2})\right), 6D^2L^2\right).
\end{align*}
\paragraph{Regret bound.} Substituting the above bounds for $S_1,S_2$ in the regret bound for $\x$ player gives us the following bound, which holds with probability at least $1-\delta/2$
\begin{align*}
    \sum_{t=1}^Tf(P_t,Q_t) - f(\x,Q_t) &\leq  \eta D\log{d} + \order{GD\sqrt{\frac{T\log{\frac{8}{\delta}}}{m}}+\frac{d^2D^3L^2T}{\eta m}} \\
    &\quad +\order{\min\left(\frac{d^3DG^2T}{\eta m}\left(\log{\frac{8}{\delta}}+d\log(2m)\right), \frac{d^2D^3L^2T}{\eta}\right)}\\
    &\quad+\sum_{t=1}^T\frac{3cd^2DL^2}{2\eta}\gammaFp(Q_t^{\infty},\Tilde{Q}_{t-1}^{\infty})^2- \sum_{t=1}^T\frac{\eta}{2cd^2D}\gammaF(P_t^{\infty},\Tilde{P}_{t-1}^{\infty})^2
\end{align*}
Using a similar analysis, we get the following regret bound for the $\y$ player
\begin{align*}
    \sum_{t=1}^Tf(P_t,Q_t) - f(\x,Q_t) &\leq  \eta D\log{d} + \order{GD\sqrt{\frac{T\log{\frac{8}{\delta}}}{m}}+\frac{d^2D^3L^2T}{\eta m}} \\
    &\quad +\order{\min\left(\frac{d^3DG^2T}{\eta m}\left(\log{\frac{8}{\delta}}+d\log(2m)\right), \frac{d^2D^3L^2T}{\eta}\right)}\\
    &\quad+\sum_{t=1}^T\frac{3cd^2DL^2}{2\eta}\gammaF(P_t^{\infty},\Tilde{P}_{t-1}^{\infty})^2- \sum_{t=1}^T\frac{\eta}{2cd^2D}\gammaFp(Q_t^{\infty},\Tilde{Q}_{t-1}^{\infty})^2
\end{align*}
Choosing, $\eta = 10d^2D(L+1), m= T$, and adding the above two regret bounds, we get
\begin{align*}
 \sup_{\x\in\cX,\y\in\cY}\sum_{t=1}^Tf(P_t,\y) - f(\x,Q_t)& = \order{d^2D^2(L+1)\log{d} + GD\sqrt{\log{\frac{8}{\delta}}}}\\
 &\quad + \order{\min\left\lbrace D^2LT, \frac{d^2G^2\log{T}}{L} + \frac{dG^2\log{\frac{8}{\delta}}}{L}\right\rbrace}.
\end{align*}
\end{proof}
\section{Background on Convex Analysis}
\paragraph{Fenchel Conjugate.}
\label{sec:fenchel_conjugate}
The Fenchel conjugate of a function $f$ is defined as
\[
f^*(x^*) = \sup_{x}\iprod{x}{x^*} - f(x).
\]
We now state some useful properties of Fenchel conjugates. These properties can be found in~\citet{rockafellar1970convex}.
\begin{theorem}
\label{thm:fenchel_prop1}
Let $f$ be a proper convex function. The conjugate function $f^*$ is then a closed and proper convex function. Moreover, if $f$ is lower semi-continuous then $f^{**} = f$.
\end{theorem}
\begin{theorem}
\label{thm:fenchel_prop3}
For any proper convex function $f$ and any vector $x$, the following conditions on a vector $x^*$  are equivalent to each other
\begin{itemize}
    \item $x^* \in \partial f(x)$
    \item $\iprod{z}{x^*} - f(z)$ achieves its supremum in $z$ at $z=x$
    \item $f(x) + f^*(x^*) = \iprod{x}{x^*}$
\end{itemize}
If $(\text{cl}f)(x) = f(x)$, the following condition can be added to the list
\begin{itemize}
    \item $x\in\partial f^*(x^*)$
\end{itemize}
\end{theorem}
\begin{theorem}
\label{thm:fenchel_prop4}
If $f$ is a closed proper convex function, $\partial f^*$ is the inverse of $\partial f$ in the sense of multivalued mappings, \emph{i.e.,} $x \in \partial f^*(x^*)$ iff $x^* \in \partial f(x).$
\end{theorem}
\begin{theorem}
\label{thm:fenchel_prop2}
Let $f$ be a closed proper convex function. Let $\partial f$ be the subdifferential mapping. The effective domain of $\partial f$, which is the set $\dom{\partial f} = \{x|\partial f \neq 0\},$ satisfies
    \[
    \text{ri}(\dom{ f}) \subseteq \dom{\partial f} \subseteq \dom{ f}.
    \]
    The range of $\partial f$ is defined as $\text{range} \partial f=\cup\{\partial f(x)|x\in \mathbb{R}^d\}$.
    The range of $\partial f$ is the effective domain of $\partial f^*$, so
    \[
    \text{ri}(\dom{ f^*}) \subseteq \text{range} \partial f \subseteq \dom{ f^*}.
    \]
\end{theorem}
\paragraph{Strong Convexity and Smoothness.} We now define strong convexity and strong smoothness and show that these two properties are duals of each other.
\begin{defn}[Strong Convexity]
\label{def:strong_convexity}
A function $f:\cX \to \mathbb{R}\cup\{\infty\}$ is $\beta$-strongly convex w.r.t a norm $\|\cdot\|$ if for all $x,y \in \text{ri}(\dom{ f})$  and $\alpha\in (0,1)$ we have
\[
f(\alpha x + (1-\alpha)y) \leq \alpha f(x) + (1-\alpha)f(y) - \frac{1}{2}\beta \alpha (1-\alpha) \|x-y\|^2.
\]
\end{defn}
This definition of strong convexity is equivalent to the following condition on $f$~\citep[see Lemma 13 of][]{shalev2007thesis}
\[
f(y) \geq f(x) + \iprod{g}{y-x} + \frac{1}{2}\beta\|y-x\|^2, \quad \text{for any } x,y\in \text{ri}(\dom{ f}), g\in\partial f(x)
\]
\begin{defn}[Strong Smoothness]
\label{def:strong_smoothness}
A function $f:\cX \to \mathbb{R}\cup\{\infty\}$ is $\beta$-strongly smooth w.r.t a norm $\|\cdot\|$ if $f$ is everywhere differentiable and if for all $x,y$ we have
\[
f(y) \leq f(x) + \iprod{\grad f(x)}{y-x} + \frac{1}{2}\beta\|y-x\|^2.
\]
\end{defn}
\begin{theorem}[\citet{kakade2009duality}]
\label{thm:fenchel_strong_convex_weak}
Assume that $f$ is a proper closed and convex function. Suppose $f$ is $\beta$-strongly smooth w.r.t a norm $\|\cdot\|$. Then its conjugate $f^*$ satisfies the following for all $a,x$ with $u = \grad f(x)$
\[
f^*(a+u)  \geq f^*(u) + \iprod{x}{a}+\frac{1}{2\beta}\|a\|_*^2.
\]
\end{theorem}
\begin{theorem}[\citet{kakade2009duality}]
\label{thm:fenchel_strong_convex}
Assume that $f$ is a closed and convex function.
Then $f$ is $\beta$-strongly convex w.r.t a norm $\|\cdot\|$ iff $f^*$ is $\frac{1}{\beta}$-strongly smooth w.r.t the dual norm $\|\cdot\|_{*}$.
\end{theorem}
 \end{document}